\documentclass{article}
\usepackage{kr}

\usepackage{times}
\usepackage{soul}
\usepackage{url}
\usepackage[hidelinks]{hyperref}
\usepackage[utf8]{inputenc}
\usepackage[small]{caption}
\usepackage{graphicx}
\usepackage{amsmath}
\usepackage{amsthm}
\usepackage{booktabs}
\usepackage{algorithm}
\usepackage[noend]{algorithmic} 
\usepackage[switch]{lineno}

\usepackage{enumitem}
\setlist[itemize]{noitemsep, nolistsep}
\usepackage{hyperref}
\hypersetup{
    colorlinks=true,
    linkcolor=black,
    filecolor=black,      
    urlcolor=black,
    citecolor=.
    }

\newtheorem{theorem}{Theorem}

\author{
Giovanni Varricchione$^1$\and
Toryn Q. Klassen$^{2, 3}$\and
Natasha Alechina$^{4,1}$\and\\
Mehdi Dastani$^1$\and
Brian Logan$^{5,1}$\And
Sheila A. McIlraith$^{2, 3}$\\
\affiliations
$^1$Utrecht Universiteit, Utrecht, The Netherlands\\
$^2$University of Toronto, Toronto, Canada\\
$^3$Vector Institute, Toronto, Canada\\
$^4$Open Universiteit, Heerlen, The Netherlands\\
$^5$University of Aberdeen, Aberdeen, United Kingdom\\
\emails
\{g.varricchione, n.a.alechina, m.m.dastani, b.s.logan\}@uu.nl,
\{toryn, sheila\}@cs.toronto.edu
}

\usepackage[english]{babel}
\usepackage{subcaption}
\usepackage{url}
\usepackage{colortbl}
\usepackage{amssymb}
\usepackage{stmaryrd}
\usepackage{amsmath}
\usepackage{mathrsfs}
\usepackage{amsthm}
\newtheorem{proposition}{Proposition}
\theoremstyle{definition}
\newtheorem{definition}{Definition}

\newtheorem*{problems*}{Learning Problems}
\usepackage{nicefrac}
\usepackage[textsize=scriptsize,backgroundcolor=yellow!40,obeyDraft=true]{todonotes}
\usepackage[nonumberlist,acronym]{glossaries}
\glsdisablehyper
\usepackage{comment}
\usepackage[capitalise,nameinlink]{cleveref}
\crefname{algocf}{alg.}{algs.}
\Crefname{algocf}{Algorithm}{Algorithms}
\usepackage{tkz-base}
\usetikzlibrary{decorations.pathmorphing,trees,decorations,arrows,shapes,automata,petri,fit}
\usepackage{multicol}
\usepackage{booktabs}
\usepackage{enumitem}
\usepackage{multirow}
\usepackage{rotating}
\usepackage{etaremune}
\usepackage[ruled,linesnumbered,algo2e]{algorithm2e}
\usepackage{marginnote}
\usepackage{mathtools}
\usepackage{adjustbox}
\usepackage{ifthen}
\usepackage[normalem]{ulem}
\usepackage{lineno}
\usepackage{soul}
\usepackage{listings}
\crefname{lstlisting}{Listing}{Listings}
\usepackage{lipsum}
\usepackage{makecell}
\usepackage{diagbox}
\usepackage[scientific-notation=false,group-separator={,}]{siunitx}
\usepackage{microtype}
\usepackage[bottom]{footmisc}
\usepackage{xspace}
\usepackage{xparse}
\usepackage{xcolor}
\usepackage{stmaryrd}
\SetSymbolFont{stmry}{bold}{U}{stmry}{m}{n}
\usepackage{bbold}

\newcommand{\ie}{i.e.,\ }

\newcommand{\citet}[1]{\citeauthor{#1} (\citeyear{#1})}
\def\la {\langle}
\def\ra {\rangle}
\newcommand{\tuple}[1]{\langle{#1}\rangle}

\newcommand{\ProbDistrs}{\Delta}

\def\Reals {\mathbb{R}}

\newcommand{\MDP}{\ensuremath{M}}

\newcommand{\MDPStates}{\ensuremath{S}}
\newcommand{\MDPState}{\ensuremath{s}}

\newcommand{\MDPActions}{\ensuremath{A}}
\newcommand{\MDPAction}{\ensuremath{a}}

\newcommand{\MDPTransition}{\ensuremath{p}}
\newcommand{\MDPReward}{\ensuremath{r}}
\newcommand{\MDPDiscount}{\gamma}

\newcommand{\Policy}{\pi}

\newcommand{\qapprox}{\tilde{q}}

\newcommand{\RM}{\mathcal{R}}

\newcommand{\RMStates}{\ensuremath{U}}
\newcommand{\RMState}{\ensuremath{u}}
\newcommand{\RMInState}{u_0}
\newcommand{\RMFinStates}{\ensuremath{F}}
\newcommand{\RMAlphabet}{\Sigma}
\newcommand{\RMAlphabetSymbol}{\sigma}
\newcommand{\RMTransition}{\delta_u}
\newcommand{\RMReward}{\delta_r}

\newcommand{\Labelling}{\ensuremath{L}}

\newcommand{\pdRM}{pdRM\xspace}
\newcommand{\pdRMs}{pdRMs\xspace}
\newcommand{\PRMStackAlphabet}{\Gamma}
\newcommand{\PRMStackInSymbol}{\ensuremath{Z}}
\newcommand{\PRMStackSymbol}{\ensuremath{z}}
\newcommand{\PRMStackString}{\zeta}

\newcommand{\PRMConfigTransition}[1]{\vdash_{#1}}

\newcommand{\MDPpdRM}{MDP-pdRM\xspace}
\newcommand{\MDPpdRMs}{MDP-pdRMs\xspace}
\newcommand{\mMDPpdRM}{\mathcal{T}}

\newcommand{\CpdRM}{CpRM\xspace}

\title{Pushdown Reward Machines for Reinforcement Learning}

\begin{document}


\maketitle

\begin{abstract}

Reward machines (RMs) are automata structures that encode (non-Markovian) reward functions for reinforcement learning (RL). RMs can reward any behaviour representable in regular languages and, when paired with RL algorithms that exploit RM structure, have been shown to significantly improve sample efficiency in many domains. In this work, we present \emph{pushdown reward machines} (\pdRMs), an extension of reward machines based on deterministic pushdown automata. \pdRMs can recognise and reward temporally extended behaviours representable in deterministic context-free languages, making them more expressive than reward machines. We introduce two variants of \pdRM-based policies, one which has access to the entire stack of the \pdRM, and one which can only access the top $k$ symbols (for a given constant $k$) of the stack. We propose a procedure to check when the two kinds of policies (for a given environment, \pdRM, and constant $k$) achieve the same optimal state values. We then provide theoretical results establishing the expressive power of pdRMs, and space complexity results for the proposed learning problems. Lastly, we propose an approach for off-policy RL algorithms that exploits counterfactual experiences with \pdRMs. We conclude by providing experimental results showing how agents can be trained to perform tasks representable in deterministic context-free languages using pdRMs.

\end{abstract}

\section{Introduction}
\label{sec:introduction}
Reward machines (RMs)~\cite{Toro-icml18,Toro-jair22} are automata structures that are used to represent (non-Markovian) reward functions for reinforcement learning (RL). Among their merits, they enable RL algorithms to exploit the compositional structure of RMs in learning, resulting in significant sample efficiency gains. By virtue of their correspondence to deterministic finite state automata (DFAs), any reward-worthy behaviour expressible by a regular language, as well as variants of other formal languages such as variants of linear temporal logic (LTL), can be encoded by an RM. This means that a human can write their non-Markovian reward function, or reward-worthy (temporally extended) behaviour in a diversity of programming/formal languages, compile them to an RM, and take advantage of the sample efficiency gains of these RM-tailored learning algorithms~\cite{Camacho//:19}. 

A restriction of RMs is that reward-worthy behaviour must be representable in a DFA-like structure---a regular or Type-3 language, according to Chomsky's hierarchy of languages~\cite{Chomsky//:59}. However, a number of interesting RL problems require the expressiveness of a context-free language.
To enhance the expressiveness of RMs, \citet{Bester//:24} introduced \emph{counting reward automata} (CRAs) which augment RMs with counters. CRAs have the expressive power of counter machines.
As a counter machine with two or more counters has the same expressive power as a Turing machine \cite{Minsky//:67}, CRAs with two or more counters can express behaviours that are representable by recursively enumerable languages, the largest class in the Chomsky hierarchy. 
Unfortunately, the expressive power of CRAs can come at a significant computational cost for RL. The resulting product MDP and policy incur a blowup depending on the maximum values that the counters can assume. This blowup can severely hinder training by slowing the convergence speed: as the MDP and policy state spaces grow, the time required to explore and learn increases. 

In this paper, following Chomsky's hierarchy, we propose a more modest enhancement to the expressiveness of RMs by augmenting RMs with a single stack. We call these enhanced RMs \emph{pushdown reward machines} (\pdRMs). Our enhancement is based on deterministic pushdown automata (DPDAs), which allow us to encode reward-worthy behaviours that are representable, for example, in LR grammars \cite{Knuth//:65}, or, precisely, deterministic context-free languages. \pdRMs can recognize a wide range of practical behaviours, such as modelling recursive calls in programming, collecting and delivering arbitrary numbers of parcels to specific locations, or search and rescue tasks where an agent must return to its starting point by remembering and retracing its (safe) route.

The main contributions of this paper are as follows: 
\begin{itemize}
    \item We define reward machines based on deterministic pushdown automata. Given their structure, \pdRMs can encode tasks representable in deterministic context-free languages;
        \item We define two variants of \pdRM-based policies. In the first, the policy has access to the entire \pdRM stack, in the second it can only access the top $k$ symbols (for a given $k$) of the stack;
        \item We provide a procedure to check whether optimal policies with access to the top $k$ stack symbols achieve the same state values as optimal policies with full-stack access;
        \item We analyse the expressive power of \pdRMs and compare it to RMs and CRAs. We also evaluate the space blowup for \pdRM- and CRA-based policies. We show that \pdRM-policies accessing only the top $k$ symbols of the stack are more compact than \pdRM-policies accessing the entire stack and CRA-based policies;
        \item We propose an approach that exploits counterfactual experiences by generating synthetic experiences based on the states of the \pdRM and the stack strings observed during training;
        \item We use \pdRMs to train RL agents in several domains. We compare them to CRAs in a domain from \cite{Bester//:24}. Then, we show the practical effects of the space complexity results we establish. We show how counterfactual and hierarchical approaches for \pdRMs can be used to increase sample efficiency. Finally, we use \pdRMs in a continuous domain, and compare against a deep learning algorithm using recurrent neural networks.
 \end{itemize}      

\section{Preliminaries}
\label{sec:preliminaries}
In reinforcement learning (RL), the environment in which agents act and learn is modelled as a \emph{Markov decision process} (MDPs) \cite{puterman2014markov}.
A Markov decision process is a tuple $\MDP = \la \MDPStates, \MDPActions, \MDPTransition, \MDPReward, \MDPDiscount \ra$ where $\MDPStates$ is the non-empty set of states, $\MDPActions$ is the non-empty set of actions, $\MDPTransition: \MDPStates \times \MDPActions \to \ProbDistrs(\MDPStates)$ is the state transition function where $\ProbDistrs(\MDPStates)$ is the set of all probability distributions defined over $\MDPStates$, $\MDPReward: \MDPStates \times \MDPActions \times \MDPStates \to \Reals$ is the reward function, and $\MDPDiscount \in (0, 1)$ is the discount factor. 
We write $\MDPTransition(\MDPState' \mid \MDPState, \MDPAction)$ to denote the probability of transitioning from state $\MDPState$ to state $\MDPState'$ when action $\MDPAction$ is performed in $\MDPState$.
A \emph{policy} is a function $\Policy: \MDPStates \to \ProbDistrs(\MDPActions)$ mapping any state of the MDP to a probability distribution over the set of actions. We denote by $\Policy(\MDPAction \mid \MDPState)$ the probability that an agent following policy $\Policy$ performs action $\MDPAction$ in state $\MDPState$.

At each timestep $t$, the MDP is in some state $\MDPState_t$. The agent takes an action $\MDPAction_t \sim \Policy(\cdot \mid \MDPState_t) $ using its policy $\Policy$, after which the environment state is updated to $\MDPState_{t + 1} \sim \MDPTransition(\cdot \mid \MDPState_t, \MDPAction_t)$ and the agent is rewarded with $r_t = \MDPReward(\MDPState_t, \MDPAction_t, \MDPState_{t + 1})$. 
The goal of the agent is to learn an \emph{optimal} policy $\pi^*$, \ie one that maximizes the expected discounted reward $\mathbb{E}_{\pi^*} \left[ \sum_{k = 0}^\infty \gamma^k r_k \mid S_0 = s\right]$ from any MDP state $\MDPState$.

Reward machines were introduced by \citeauthor{Toro-jair22} (\citeyear{Toro-icml18,Toro-jair22})
to encode non-Markovian reward functions. A reward machine (RM) is a tuple $\RM = \la \RMStates, \RMInState, \RMFinStates, \RMAlphabet, \RMTransition, \RMReward \ra$, where $\RMStates$ is a finite non-empty set of states, $\RMInState$ is the initial reward machine state, $\RMFinStates \not\subseteq \RMStates$ is the set of final states, $\RMAlphabet$ is the input alphabet, $\RMTransition: \RMStates \times \RMAlphabet \to \left(\RMStates \cup \RMFinStates\right)$ is the transition function, and $\RMReward: \RMStates \times \RMAlphabet \to  \Reals$ is the output reward function.  
The transitions in the reward machine are related to transitions in the MDP by a \emph{labelling function} $\Labelling: \MDPStates \times \MDPActions \times \MDPStates \to \RMAlphabet$ which labels each state-action-state triple of the MDP with an input symbol of the reward machine.
When the reward function is specified by a reward machine, the agent's policy is defined over the Cartesian product of the set of MDP states and the set of states of the reward machine.

\section{Pushdown Reward Machines}
\label{sec:pushdown-reward-machines}
In this section, we define pushdown Reward Machines (\pdRMs). Like RMs, \pdRMs are used to express (non-Markovian) reward functions, however their automaton structure is enhanced with a stack, which serves as additional memory. In so doing, they enable the expression of reward-worthy temporally-extended behaviours that correspond to deterministic context-free languages.

\begin{definition}[Pushdown Reward Machine]
    A \emph{pushdown reward machine} (\pdRM) is a tuple $\RM = \la \RMStates, \RMInState, \RMFinStates, \RMAlphabet, \PRMStackAlphabet, \PRMStackInSymbol, \RMTransition, \RMReward \ra $, where:
    \begin{itemize}
    \item $\RMStates$ is the finite set of states; 
    \item $\RMInState \in \RMStates$ is the initial state; 
    \item $\RMFinStates \not\subseteq \RMStates$ is the set of final states; 
    \item $\RMAlphabet_\epsilon = \RMAlphabet \cup \{ \epsilon \}$ where $\RMAlphabet$ is the input alphabet;  
    \item $\PRMStackAlphabet_\epsilon = \PRMStackAlphabet \cup \{ \epsilon \}$ where $\PRMStackAlphabet$ is the stack alphabet;  
    \item $\PRMStackInSymbol \in \PRMStackAlphabet$ is the initial stack symbol;   
    \item $\RMTransition: \RMStates \times \RMAlphabet_\epsilon \times \PRMStackAlphabet_\epsilon \to \left(\RMStates \cup \RMFinStates\right) \times \PRMStackAlphabet_{\epsilon}^*$ is the transition function; and 
    \item $\RMReward: \RMStates \times \RMAlphabet_\epsilon \times \PRMStackAlphabet_\epsilon \to \Reals $ is the reward function. 
    \end{itemize}
    \end{definition}
\noindent where $\epsilon$ is the empty string. Below, we denote by $\PRMStackSymbol \in \PRMStackAlphabet$ individual symbols of the stack alphabet, and by $\PRMStackString \in \PRMStackAlphabet^*$ (possibly empty) stack strings. The transition function $\RMTransition$ takes as input the current \pdRM state $\RMState$, the current input symbol $\RMAlphabetSymbol \in \RMAlphabet_\epsilon$ and the topmost symbol of the stack $\PRMStackSymbol \in \PRMStackAlphabet_\epsilon$, and returns a pair $\left( \RMState', \PRMStackString' \right)$ where $\RMState' \in \left(\RMStates \cup \RMFinStates\right)$ is the next state and $\PRMStackString' \in \PRMStackAlphabet_{\epsilon}^*$ is a string which replaces $\PRMStackSymbol$ as the topmost symbol(s) of the stack. Note that $\RMAlphabetSymbol$, $\PRMStackSymbol$ and $\PRMStackString'$ may be the empty string $\epsilon$. Whenever $\RMAlphabetSymbol = \epsilon$, the \pdRM does not read any symbol from the input, making a so-called ``silent'' transition. If $\PRMStackSymbol = \epsilon$, no symbol is popped from the stack. Finally, if $\PRMStackString' = \epsilon$, no symbol is pushed to the stack.
In what follows, we consider only \emph{deterministic} \pdRMs where, for any \pdRM state $\RMState$ and stack symbol $\PRMStackSymbol$, if $\RMTransition(\RMState, \epsilon, \PRMStackSymbol)$ is defined, then $\RMTransition(\RMState, \RMAlphabetSymbol, \PRMStackSymbol)$ is not defined for any $\RMAlphabetSymbol \in \RMAlphabet$.
The reward function $\RMReward$ takes as input the current \pdRM state $\RMState$, the current input symbol $\RMAlphabetSymbol$ and the topmost symbol $\PRMStackSymbol$ of the stack and returns a reward.

At each timestep, a \pdRM is in a \emph{configuration} $\la \RMState, \PRMStackSymbol\PRMStackString \ra \in \RMStates \times \PRMStackAlphabet_{\epsilon}^*$, where $\RMState$ is the state of the pushdown reward machine and $\PRMStackSymbol\PRMStackString$ is the current string on the stack with $\PRMStackSymbol$ as the topmost symbol on the stack. The \pdRM reads the current input symbol $\RMAlphabetSymbol \in \RMAlphabet_\epsilon$, transitions to a new configuration $\la \RMState', \PRMStackString'\PRMStackString \ra$ where $\la \RMState', \PRMStackString'\ra = \RMTransition(\RMState, \RMAlphabetSymbol, \PRMStackSymbol)$, and outputs a reward $ r = \RMReward(\RMState, \RMAlphabetSymbol, \PRMStackSymbol)$. We write $\la \RMState, \PRMStackString \ra \PRMConfigTransition{\RMAlphabetSymbol} \la \RMState', \PRMStackString' \ra$ to denote that the \pdRM moves from $\la \RMState, \PRMStackString \ra$ to $\la \RMState', \PRMStackString' \ra$ upon reading the symbol $\RMAlphabetSymbol$.

We illustrate the expressive power of a \pdRM using a simple task expressible in a context-free language, which we call ``Maze'' task. In the task, the agent has to navigate a (gridlike) maze from a starting location to find a treasure and return to the starting point by following the path it traversed to reach the treasure in the reverse direction. 
The actions available to the agent are $\MDPActions = \{ \mathsf{u}, \mathsf{d}, \mathsf{l}, \mathsf{r} \}$, denoting $\mathsf{u}$p, $\mathsf{d}$own, $\mathsf{l}$eft, and $\mathsf{r}$ight respectively, and the task can be defined using the following (deterministic) context-free grammar:
\begin{align*}
S &\rightarrow P\,\mathsf{x}\\
P &\rightarrow \mathsf{u}\,P\,\mathsf{d} \mid \mathsf{d}\,P\,\mathsf{u} \mid \mathsf{l}\,P\,\mathsf{r} \mid \mathsf{r}\,P\,\mathsf{l} \mid \mathsf{t}
\end{align*}
where $\mathsf{t}$ denotes the treasure and $\mathsf{x}$ the starting point.

The corresponding \pdRM is shown in \Cref{fig:pdrm-maze}. The states in the \pdRM are $\RMStates = \{ u_0, u_1, u_2, u_3 \}$, the stack alphabet is $ \PRMStackAlphabet = \{ \mathsf{u}, \mathsf{d}, \mathsf{l}, \mathsf{r} \}$, the input alphabet $\RMAlphabet = 2^{\PRMStackAlphabet \cup \{ \mathsf{t}, \mathsf{x} \}}$ and the labelling $\Labelling$ is defined as follows:
    \begin{itemize}
        \item $a \in \Labelling(\MDPState, \MDPAction, \MDPState')$;
        \item $\mathsf{t} \in \Labelling(\MDPState, \MDPAction, \MDPState')$ if $\MDPState \neq \MDPState'$ and $\MDPState'$ is the location of the treasure;
        \item $\mathsf{x} \in \Labelling(\MDPState, \MDPAction, \MDPState')$ if $\MDPState \neq \MDPState'$ and $\MDPState'$ is the starting point.
    \end{itemize}

    \begin{figure}
        \centering
        \resizebox{\linewidth}{!}
        {
            \begin{tikzpicture}[shorten >=1pt,auto]
                \node[state,initial]      (u_0)                           {\LARGE $u_0$};
                \node[state]              (u_1) [right=5cm of u_0]        {\LARGE $u_1$};
                \node[state, accepting]   (u_3) [right=5cm of u_1]  {\LARGE $u_3$};
                \node[state, accepting]   (u_2) [above=3cm of u_3]  {\LARGE $u_2$};
    
                \path[->]
                    (u_0)   
                        edge [loop below]   node [align=center] {\LARGE $ \{ a \}, \PRMStackSymbol / a\ \PRMStackSymbol, 0$}    ()
                        edge    node    {\LARGE $ \{ a, \mathsf{t} \}, \PRMStackSymbol / a\ \PRMStackSymbol, 0$}  (u_1)
    
                    (u_1)
                        edge [loop below]   node [align=center] {\LARGE $\{ \overline{a} \}, a / \epsilon, 0$}    ()
                        edge    node    {\LARGE $ \{ \overline{a}, \mathsf{x} \}, a / \epsilon, 1$}  (u_3)
                        edge [bend left]   node [align=center]    {\LARGE $\{ a' \}, a / \epsilon, -1$\\\LARGE $ \{ a', \mathsf{x} \}, a / \epsilon, -1$}    (u_2);
            \end{tikzpicture}
        }
        \caption{Pushdown reward machine for the Maze task. Each transition is labelled with a tuple $\ell, \PRMStackSymbol / \PRMStackString, r$, where $\ell$ is the input observation, $\PRMStackSymbol$ the top symbol on the stack, $\PRMStackString$ the string of symbols pushed onto the stack (with the new top symbol leftmost in $\PRMStackString$), and $r$ is the output reward. The symbol ``$\PRMStackSymbol$'' indicates an arbitrary symbol in $\PRMStackAlphabet$. In the transition from $u_1$ to $u_2$, $a'$ represents any direction except for $\overline{a}$, the opposite direction to $a$. 
        }
        \label{fig:pdrm-maze}
    \end{figure}

As with standard reward machines, \pdRMs can be used to reward the agent in MDPs. The product of an MDP and a \pdRM gives rise to an ``\MDPpdRM'', defined as follows.

\begin{definition}
        A \emph{Markov decision process with a pushdown reward machine} (\MDPpdRM) is a tuple $\mMDPpdRM = \la \MDPStates, \MDPActions, \MDPTransition, \MDPDiscount, \Labelling, \RM \ra$, where $\MDPStates, \MDPActions, \MDPTransition, \MDPDiscount$ are defined as in an MDP, $\Labelling: \MDPStates \times \MDPActions \times \MDPStates \to \RMAlphabet$ is a labelling function, and $\RM = \la \RMStates, \RMInState, \RMFinStates, \RMAlphabet, \PRMStackAlphabet, \PRMStackInSymbol, \RMTransition, \RMReward \ra $ is a \pdRM.
\end{definition}

Markov decision processes with a pushdown reward machine are analogous to the \emph{Markov decision processes with a reward machine} introduced in \cite{Toro-icml18}, but with rewards specified by a \pdRM rather than an RM.

An \MDPpdRM induces a \emph{product MDP} $\MDP_\mMDPpdRM = \la \MDPStates_\mMDPpdRM, \MDPActions_\mMDPpdRM, \MDPTransition_\mMDPpdRM, \MDPReward_\mMDPpdRM, \MDPDiscount_\mMDPpdRM \ra $ where:

\begin{itemize}
    \item $\MDPStates_\mMDPpdRM := \MDPStates \times \left( \RMStates \cup \RMFinStates \right) \times \PRMStackAlphabet^*$;
    \item $\MDPActions_\mMDPpdRM := \MDPActions$;
    \item $\begin{aligned}[t]\MDPTransition_\mMDPpdRM (\la \MDPState', \RMState', &\PRMStackString' \ra \mid \la \MDPState, \RMState, \PRMStackString \ra, \MDPAction) := \\&\begin{cases}
        \MDPTransition(\MDPState' \mid \MDPState, \MDPAction) & \text{if } \RMState \in \RMFinStates, \RMState = \RMState' \text{ and } \PRMStackString = \PRMStackString'\\
        \MDPTransition(\MDPState' \mid \MDPState, \MDPAction) & \text{if } \RMState \not\in \RMFinStates \text{ and }\\& \la \RMState, \PRMStackString \ra \PRMConfigTransition{\Labelling( \MDPState, \MDPAction, \MDPState' )} \la \RMState', \PRMStackString' \ra\\
        0 & \text{otherwise}
    \end{cases}\end{aligned}$
    
    \item $\MDPReward_\mMDPpdRM \left( \la \MDPState, \RMState, \PRMStackSymbol\PRMStackString \ra, \MDPAction, \la \MDPState', \RMState', \PRMStackString' \ra \right) := \RMReward( \RMState, \Labelling( \MDPState, \MDPAction, \MDPState' ), \PRMStackSymbol )$;
    \item $\MDPDiscount_\mMDPpdRM := \MDPDiscount$.
\end{itemize}
Although not specified, the starting state of the \pdRM in the product MDP is always $\RMInState$.

We now define the possible policies for \MDPpdRMs. The first type of policy has full access to the current state of the product MDP, and thus to all of the stack's contents.

\begin{definition}[Policy]
    A \emph{policy} in an \MDPpdRM is a function $\Policy: \MDPStates \times \RMStates \times \PRMStackAlphabet^* \to \Delta(\MDPActions)$. 
\end{definition}

A policy has, at each timestep, access to the current MDP state, \pdRM state, and the entire contents of the \pdRM stack. Since the stack is potentially unbounded, the state space of policies (hence their size) is potentially unbounded. As such, policies are not well-suited to many RL algorithms, e.g., tabular algorithms, especially in non-episodic tasks (\ie a task in which the agent must act indefinitely in the environment). We therefore define a bounded version of policies which have access only to a portion of the stack. In this way, we bound the size of the policy, making it finite regardless of whether the task is episodic or not.

\begin{definition}[$k$-policy]
    Given a constant $k \geq 0$, a \emph{$k$-policy} is a function $\Policy: \MDPStates \times \RMStates \times \left(\bigcup_{j\le k}\PRMStackAlphabet^j\right) \to \Delta(\MDPActions)$. 
\end{definition}

While $k$-policies can be more suited to learning in RL, we note that their limited observability of the product MDP's state can lead to suboptimal behaviours. In the next section, we will provide a procedure to verify in which cases a $k$-policy has access to enough information to achieve behaviours with the same state values as optimal policies.

\section{When Are \texorpdfstring{$k$}{k}-Policies Optimal?}
\label{sec:when-can-we-learn}
As we have just observed, $k$-policies can achieve suboptimal behaviours compared to policies, since they have limited access to the product MDP's state. 
For example, imagine a task in which the agent is given a sequence of rooms (which is saved on the \pdRM's stack) that it must clean. Depending on the room, the agent needs different supplies, which are stored in a storage room, to clean it. Gathering all the needed supplies before cleaning the rooms is more efficient as the agent will not have to return to the storage room, and for a policy to anticipate which supplies are needed, it needs access to the entire stack.
However, having access to the entire stack is not always necessary to achieve an optimal behaviour. Indeed, in the Maze task it is possible to learn an optimal policy given access to only the top symbol on the stack, as the choice of direction when returning to the starting point is determined only by the top symbol.
In this section, we characterise when, for each state, the state value for optimal $k$-policies is the same as the state value for optimal policies.

First, we define the notion of \emph{$k$-equivalence} for states in the product MDP induced by an \MDPpdRM.

\begin{definition}[$k$-equivalence $\sim_k$] 
    For a product MDP obtained from an \MDPpdRM, two states $\la s, u, \zeta \ra, \la s, u, \zeta' \ra$, are $k$-equivalent, denoted by $\la s, u, \zeta \ra \sim_k \la s, u, \zeta' \ra$, if and only if the top $k$ symbols of $\zeta$ and $\zeta'$ are the same.
\end{definition}

$k$-equivalence is an equivalence relation which partitions the state space of the product MDP into equivalence classes, where the members of each class share the MDP state, \pdRM state, and top $k$ symbols of the \pdRM stack. 
In $k$-policies, the agent's policy is defined over the set of equivalence classes of the $\sim_k$ relation. We can therefore check whether optimal $k$-policies have the same value as optimal policies. 

The size of the product MDP state space is potentially countably infinite since the states of the product contain the stack, and its size is unbounded.
However the rewards are bounded for all states and actions, and depend only on the the \pdRM state, the current input symbol and the topmost symbol on the stack. As a result, the potentially unbounded growth in the size of the stack does not affect the immediate rewards. The Bellman operator therefore remains a contraction mapping and value iteration converges to the optimal value function as in the case of finite state spaces (see, e.g., \cite{Feinberg2011}).
\begin{proposition}\label{prop:top-k}
If after performing value iteration on the product MDP, all $\sim_k$-equivalent states have the same value and the same sets of optimal actions, then optimal $k$-policies achieve the same state values as optimal policies. 
\end{proposition}
\begin{proof}
If the condition in the proposition holds, there exists an optimal policy where action selection depends only on the top $k$ symbols on the stack. Hence, there exists a $k$-policy that results in the same state values as an optimal policy. 
\end{proof}
In practice, it is not possible to check whether the condition in Proposition~\ref{prop:top-k} holds for an infinite MDP. However, by bounding the maximal stack size to, e.g., the maximal length of a stack string given the episode length, we can check whether the condition holds for the resulting finite MDP.

\section{Comparison With CRAs}
\label{sec:theoretical-results}
In this section, we analyse the relative expressive power of \pdRMs and counting reward automata (CRA) \cite{Bester//:24} and evaluate the space blowup for \pdRM- and CRA-based policies. CRAs are based on counter machines (CM), \ie finite automata augmented with $k$ counters (for given $k$) that may be incremented, decremented and tested for zero. 

First, we consider the expressive power of \pdRMs and CRAs. 
As both \pdRMs and CRAs are automata, the reward functions that they can encode correspond to the languages they accept. Deterministic pushdown automata are strictly more expressive than counter machines with one counter \cite{Fischer//:66}. On the other hand, counter machines with 2 or more counters can simulate a Turing machine \cite{Minsky//:67}, and as such they are strictly more expressive than \pdRMs with one stack. 

Next, we consider the space complexity of \pdRMs and CRAs, where by ``space complexity'' we mean the bounds on the sizes of policies, \ie the size of the table representing the polices that are learnt with \pdRMs and CRAs. As \pdRMs and CRAs only augment the state space of the underlying MDP, we analyse the blowup they cause in the state space to determine the blowup in the size of the policies. 
To compare space complexity, we focus on tasks representable in deterministic context-free languages and consider only episodic tasks of length $n$, as, in the case of non-episodic tasks (\ie tasks where the agent acts in the environment indefinitely), both stacks and counters can assume infinitely many values and policies learnt with both \pdRMs and CRAs may be infinite.

\subparagraph{Policies} 
    As policies have access to the entire stack, they incur a blowup depending on the number of possible strings that can appear on the stack (\ie the cardinality of the \textit{stack language}).
    Let $m$ be the maximum number of symbols that can be added to the stack at any transition of the \pdRM. We define an $\epsilon$-sequence as a sequence of $\epsilon$-transitions in the \pdRM which are taken before the \pdRM must read a symbol to advance. Let $e$ be the maximum number of $\epsilon$-transitions pushing symbols to the stack in a single $\epsilon$-sequence of the \pdRM. At each step, the \pdRM can add at most $m(e + 1)$ symbols to the stack, and, for input words of length $n$, the maximum length of the stack string is bounded by $nm(e + 1)$. 
    Hence, the cardinality of the stack language is $|\PRMStackAlphabet|^0 + |\PRMStackAlphabet|^1 + \ldots + |\PRMStackAlphabet|^{nm(e+1)}$. 
    If $|\PRMStackAlphabet| > 1$, then the number of stack strings is exactly $\frac{|\PRMStackAlphabet|^{nm(e+1) + 1} -1}{|\PRMStackAlphabet| - 1} \in O\left(|\PRMStackAlphabet|^{nm(e+1)}\right)$; for $|\PRMStackAlphabet| = 1$ we have that the number of stack strings is $nm(e + 1) + 1 \in O(nm(e + 1))$.  
    
    Thus, we have the following: 
    
    \begin{theorem}\label{thm:full-mdp-pdrm-space-blowup}
        If $|\PRMStackAlphabet| \geq 2$, policies incur a blowup in $O\left(|\PRMStackAlphabet|^{nm(e+1)}\right)$. If $|\PRMStackAlphabet| =1$, policies incur a blowup in $O\left({nm(e+1)}\right)$.
    \end{theorem}

\subparagraph{$k$-policies}
    Unlike for policies, for $k$-policies we just have to evaluate the number of strings of length at most $k$ made of symbols from the stack alphabet, regardless of the task encoded by the \pdRM. Clearly, the number of stack strings of length up to $k$ is $|\PRMStackAlphabet|^0 + |\PRMStackAlphabet|^1 + \ldots + |\PRMStackAlphabet|^k$. Then, similarly to policies, if $|\PRMStackAlphabet| > 1$, the number of stack strings is exactly $\frac{|\PRMStackAlphabet|^{k + 1} -1}{|\PRMStackAlphabet| - 1} \in O(|\PRMStackAlphabet|^k)$, and for $|\PRMStackAlphabet| = 1$ the number of stack strings is $k + 1 \in O(k)$.  
    Note that for $k = 1$, this implies that the blowup is linear in $|\PRMStackAlphabet|$. Hence, we get the following.

    \begin{theorem}\label{thm:k-stack-mdp-pdrm-space-blowup}
        If $|\PRMStackAlphabet| \geq 2$, $k$-policies incur a blowup in $O\left( |\PRMStackAlphabet|^k\right)$. If $|\PRMStackAlphabet| = 1$, $k$-policies incur a blowup in $O(k)$.
    \end{theorem}
            
\paragraph{Counting Reward Automata policies}

    A policy trained with access to a CRA (\emph{CRA policy}) has access to the values stored in all counters at each timestep. To evaluate the blowup in the size of the CRA policy, we analyse the number of possible combinations of counter values, which we call ``counter configurations''. 

    First of all, for any DPDA recognising some language $L$, it is possible to define a CM that recognises the same language by simulating the DPDA. The set of states and the state-transition function of the CM are the same as those of the DPDA. In order to simulate the stack, we need to encode each possible stack string in a counter configuration. 
    As we noted in the discussion of \MDPpdRM policies above, the cardinality of the stack language, and thus the number of counter configurations needed, is in $O(|\PRMStackAlphabet|^{nm(e+1)})$ (where $m$ and $e$ are as in the argument for \cref{thm:full-mdp-pdrm-space-blowup}). This implies the following. 

    \begin{theorem}\label{thm:cra-space-blowup-big-O}
        For any task representable as a DCFL, there always exists a CRA which encodes it and incurs a blowup of $O\left(|\PRMStackAlphabet|^{nm(e+1)}\right)$ in the size of the CRA policy.
    \end{theorem}
    
    The above result gives an upper bound on the number of possible counter combinations. We note that there is no general tight lower bound on the number of counter configurations for counter machines with $k$ counters ($k$-CM) to recognise arbitrary DCFLs. Below, we give an example of a DCFL for which any $k$-CM recognising it requires at least exponentially many (in the input's length) counter configurations.

    Let $L_p = \{ \sigma x \sigma^R \mid \sigma \in \{0, 1\}^*\}$ 
    be a ``marked palindrome'' language, where $\sigma^R$ is the string $\sigma$ but reversed. We show that, for any $k$-CM that recognises $L_p$, the number of counter configurations is at least exponential in the length of the input word. We recall an argument given in the proof of Theorem 1.3 in \cite{Fischer//:68}. For any $k$-CM to recognise $L_p$, it must be the case that after reading the first half of the word (\ie before the ``mark'' $x$), the $k$-CM must be in different counter configurations after reading two distinct strings $\sigma_1, \sigma_2$ of length $m$. This can happen if and only if the number of possible counter configurations (after reading words of length $m$) is in $\Omega(2^m)$, as we need to encode each string in a tuple of natural numbers. However, in the (worst) case that the string does not contain the mark $x$, the CM must encode the entire string of length $n$ in a combination of counter values. Therefore, in this second case the number of counter configurations is in $\Omega(2^n)$. Thus, we obtain the following result on the size of CRA policies.
    
    \begin{theorem}\label{thm:cra-space-blowup-omega}
    There is a task representable by a DCFL for which any CRA that encodes it incurs a blowup of $\Omega(2^n)$ in the size of the CRA policy.
    \end{theorem}
    
    In summary, \pdRMs can incur an exponential, with respect to the episode's length, blowup in the size of policies. On the other hand, for $k$-policies the blowup is polynomial (and linear when $k = 1$) with respect to the cardinality of the stack alphabet. For CRAs, we can always obtain an upper bound for any task expressible as a DCFL by simulating a \pdRM for the task (\cref{thm:cra-space-blowup-big-O}). 
    
    We have seen a case in which a CRA policy incurs a blowup which is at least exponential in the length of the episode (\cref{thm:cra-space-blowup-omega}). However, in principle there is no tight lower bound on the blowup for policies based on CRAs, and \pdRMs. For example, it is easy to see that the number of strings in the stack language of a \pdRM which recognises the language $\{ 0^n 1^n \mid n \in \mathbb{N}\}$ is linear in the length of the input word. Similarly, one can prove the same bound for CRAs. Therefore, in this case, both policies and CRA policies would incur only a linear (in the episode's length) blowup in the policies' size. Because of this, whether CRAs or \pdRMs hould be used to encode a given task should be decided on a case-by-case basis. However, if $k$-policies (for a given $k$) are sufficient to learn optimal behaviours with respect to the task (see \cref{sec:when-can-we-learn}), then $k$-policies should be preferred to policies if the policy size is to be minimised.

\section{Exploiting Counterfactual Experiences}
\label{sec:counterfactual}
RMs were designed, primarily, with two objectives in mind: to provide a normal-form representation for temporally extended reward functions specified natively in RMs or translated from various languages, and to improve sample efficiency of RL via exploitation of the structure of that normal-form representation. The latter was achieved through the development of off-policy RL algorithms that performed counterfactual reasoning over the automata structure to consider experiences in different automata states. This approach was born out in algorithms such as QRM and CRM which operated in tabular and deep learning settings~\cite{Toro-icml18,Toro-jair22}, with identical behaviour in tabular domains.
At each timestep, these algorithms augment training with the \emph{real} experience with a set of additional synthetic \emph{counterfactual} experiences that correspond to the same $(s,a,s')$ transition experienced in a counterfactual (different) state of the RM. A modified version of this algorithm and approach was proposed by \citet{Bester//:24} for use with CRAs.

In this section, we introduce an extension of CRM, which we dub \emph{\CpdRM}, that can be used with \pdRMs. To exploit pdRMs, in addition to considering a counterfactual state of the \pdRM, \CpdRM~also changes the string on the stack to produce the counterfactual experiences. However, as we noted in \cref{sec:theoretical-results}, the number of stack strings can be exponential in the length of the episode. In practice, this can make \CpdRM~extremely slow given the number of counterfactual updates it would make at each timestep. Moreover, observe that in some cases it is even possible that some stack strings are never observed. Because of this, we designed \CpdRM~so that it uses only the set of observed stack strings to produce the counterfactual experiences. 

A version of \CpdRM~for policies is shown in \cref{alg:counterfactual-pdRM-algorithm} (the approach can be adapted to $k$-policies by modifying  action sampling and policy updates; see \cref{sec:Appendix-D} in the Appendix). Compared to CRM, the main changes in \CpdRM~lie in lines 2, 6, and 10. In line 2, we initialise the set of observed stack strings $\mathcal{O}$; in line 6 we add the current stack string $\zeta$ to $\mathcal{O}$ (obviously, if $\zeta \in \mathcal{O}$ already then this operation changes nothing); finally, in line 10 we cycle through the observed stack strings, and use these to generate the counterfactual experiences. The rest of the algorithm is effectively identical to the original CRM, making it usable with any off-policy RL algorithm. In \cref{sec:experimental-evaluation} we provide results for agents trained with Q-learning augmented with \CpdRM.

\begin{algorithm}[tb]
\caption{Q-learning with \CpdRM~(policy)} \label{alg:cpdrm}
\label{alg:counterfactual-pdRM-algorithm}
{\begin{flushleft}{\bfseries Input:} MDP-pdRM $\mMDPpdRM$, num\_episodes\end{flushleft}}
\begin{algorithmic}[1]
    \STATE Initialise $\qapprox(s, u , \zeta, a)$ arbitrarily for each $\MDPState \in \MDPStates, \RMState \in \RMStates, \PRMStackString \in \PRMStackAlphabet^*, a \in A$
    \STATE $\mathcal{O} \gets \{\}$ \COMMENT{Set of observed stack strings}
    \FOR{$\ell\leftarrow 0$ \TO num\_episodes} 
        \STATE $s \gets$ EnvInitialState(), $u \gets u_0$ and $\PRMStackString \gets \PRMStackInSymbol$
        \WHILE{$s$ is not terminal \AND $u \not\in F$} 
            \STATE $\mathcal{O} \gets \mathcal{O} \cup \{ \PRMStackString\}$
            \STATE Sample action $a$ using policy derived from $\qapprox$ (e.g., $\epsilon$-greedy) given current state $\tuple{s,\PRMStackString, u}$
            \STATE Take action $a$ and observe the next state $s'$
            \FOR{pdRM state $\RMState_c \in \RMStates$}
                \FOR{stack string $z_c\zeta_c \in \mathcal{O}$}
                    \STATE $\tuple{\RMState'_c, \PRMStackString'_c} \gets \RMTransition(\RMState_c, L(s, a, s'), \PRMStackSymbol_c)$
                    \STATE $r_c \gets \RMReward(u_c, L(s, a, s'), z_c)$
                    \IF{$s'$ is terminal \OR $u_c' \in F$} 
                        \STATE $\qapprox(s,u_c,z_c\zeta_c,a) \xleftarrow{\alpha} r_c$
                    \ELSE
                        \STATE $\qapprox(s,u_c,z_c\zeta_c,a) \xleftarrow{\alpha} r_c +\ $\\
                        $\qquad\qquad\qquad \gamma \max_{a'\in A}{\qapprox(s',u_c',\zeta_c'\zeta_c,a)}$
                \ENDIF
                \ENDFOR
            \ENDFOR
            \STATE Update pdRM configuration to $\tuple{\RMState', \PRMStackString'} \dashv_{L(s, a, s')} \tuple{\RMState, \PRMStackString}$
        \ENDWHILE
    \ENDFOR
\end{algorithmic}
\end{algorithm}

\section{Experimental Evaluation}
\label{sec:experimental-evaluation}
In this section we provide empirical results we have obtained from various experiments in five domains.\footnote{The code is available at \\  https://github.com/giovannivarr/pushdown-reward-machines.}
One of the five domains is taken from \cite{Bester//:24} to allow comparison of \pdRMs with CRAs. All domains were implemented using the Gymnasium framework \cite{Towers//:24}. For discrete domains, we have trained agents using Q-learning \cite{Watkins-qlearning//:92}, both in its vanilla form and with counterfactual updates using \CpdRM. For the continuous domain we have trained agents with Proximal Policy Optimization (PPO) \cite{Schulman//:17} and its recurrent variant, using, respectively, the \textsc{Stable-Baselines3} \cite{Raffin//:21} and \textsc{SB3 Contrib} implementations. 

During experiments, agents were trained and then evaluated by periodically running $10$ test episodes after a domain-specific number of training episodes. For all experiments we plot the rewards obtained by the agents in the test episodes, normalised between $1$ and $-1$. Specifically, we plot the median rewards with lines, and the 25$^{\text{th}}$ and 75$^{\text{th}}$ percentiles of the rewards as shadowed areas. For details on the experimental setup, specifications of the machines used, and the \pdRMs we have used, we refer the reader to \Cref{sec:Appendix-E} of the Appendix.

    \subsection{\textsc{LetterEnv}}\label{subsec:letterenv-experiment}

        \begin{figure}
            \centering
            \includegraphics[width=\linewidth]{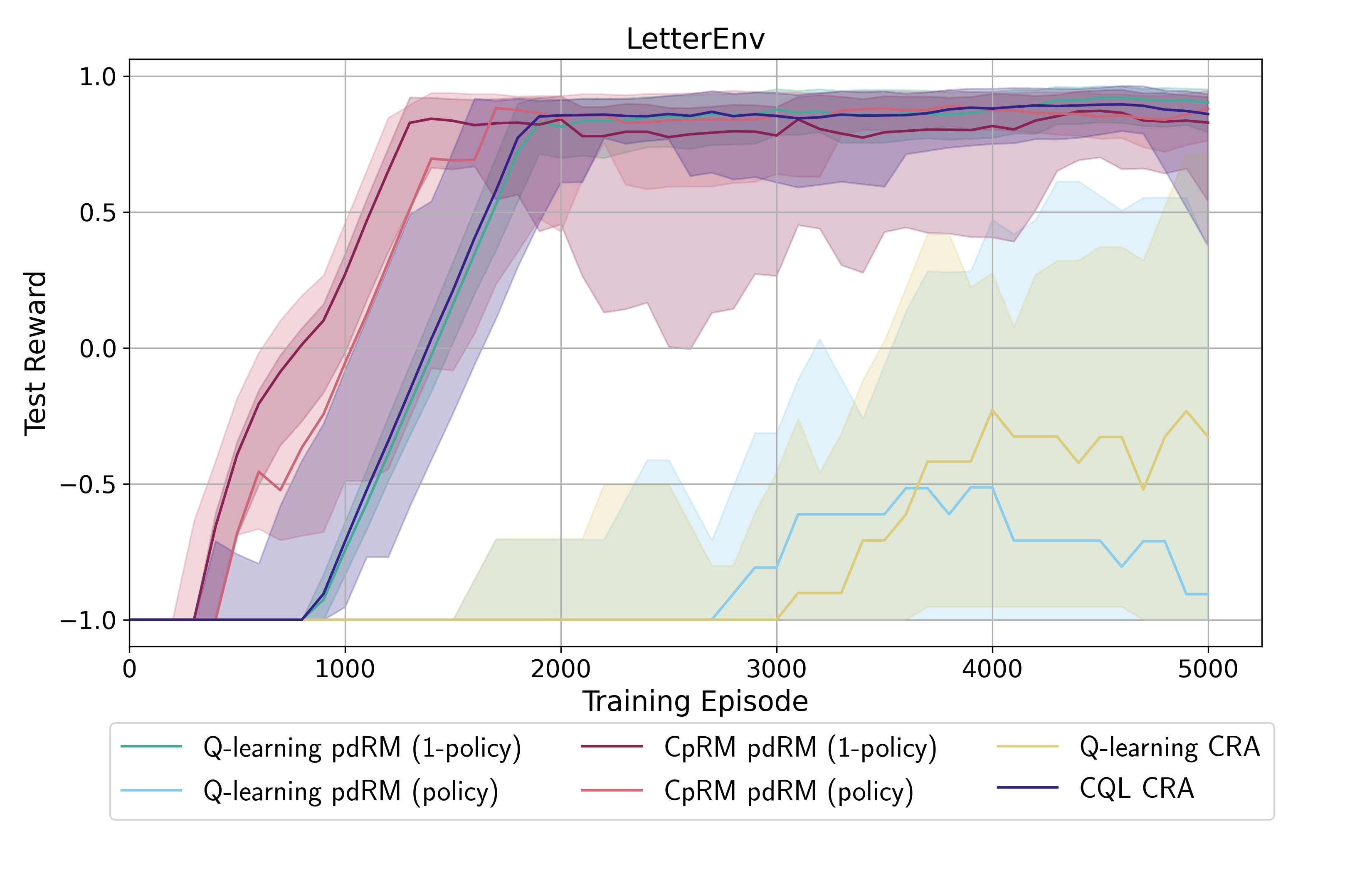}
            \caption{\textsc{LetterEnv} results, comparing agents trained with a \pdRM~and with a CRA. This shows how \pdRMs can be used to encode part of the tasks encodable by CRAs.}
            \label{fig:letterenv-experiment}
        \end{figure}

\begin{figure*}[tb]      
\centering
        \includegraphics[width=\linewidth]{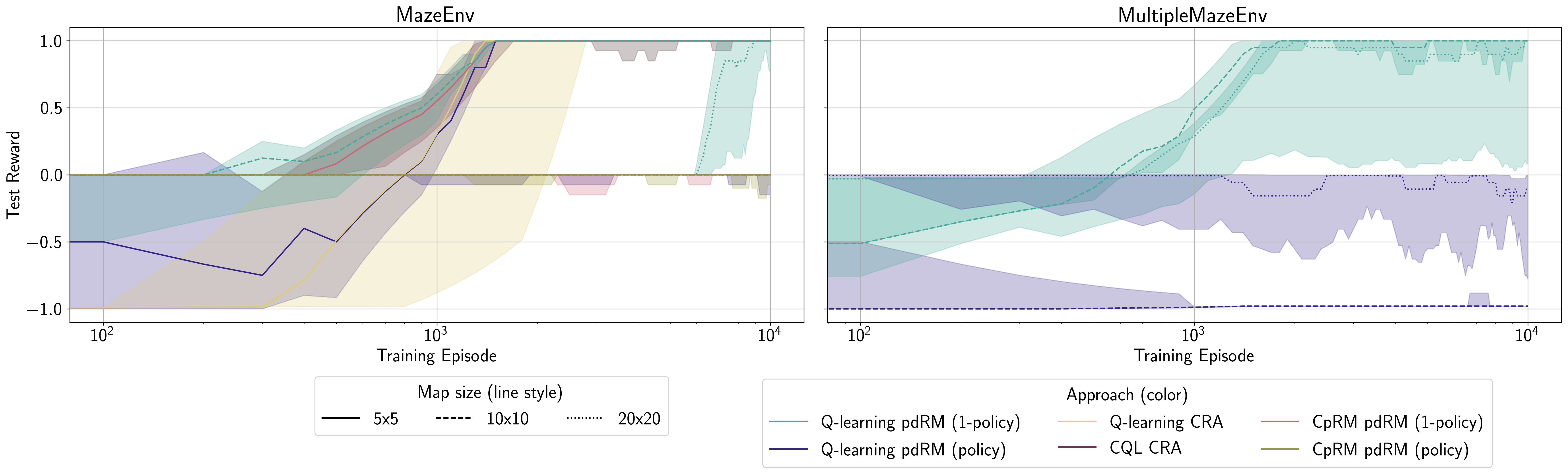}
        \caption{\textsc{1-TreasureMaze} (left) and \textsc{MultipleTreasureMaze} (right) results. We provide individual plots for each maze in \Cref{sec:Appendix-E} of the Appendix. In both plots, each maze is identified by a line style, and each agent by a colour. The \textsc{1-TreasureMaze} plot shows how the 1-\pdRM agent could achieve the task on all mazes, whereas the \pdRM and the Q-learning CRA agents only on the smallest maze. The -\CpdRM agents only managed to achieve the task in the smallest maze; in the other cases training timed out due to the time required to perform the policy updates. The CQL-CRA agent never managed to achieve the task. For the \textsc{MultipleTreasureMaze} experiment, we include only results from the 1-\pdRM and \pdRM agents trained with vanilla Q-learning. As can be seen, on mazes the 1-\pdRM agent learnt to achieve the task, whereas the \pdRM agent did not. We believe this is due to the smaller size of the $1$-policy.}
        \label{fig:maze-experiment}
\end{figure*}
        
        The \textsc{LetterEnv} domain is a domain introduced in the original paper on counting reward automata \cite{Bester//:24}. 
        The environment is a gridworld where the agent can observe three different events $A, B,$ and $C$ in specific cells. Initially, only the events $A$ and $C$ can be observed. Every time that the agent observes the event $A$ by visiting the corresponding cell, there is a probability of $\frac{1}{2}$ that the cell becomes labelled with the event $B$ from that timestep onwards. The task consists in repeatedly observing the event $A$ until the agent observes the event $B$, and then observing the event $C$ for the same number of times that the event $A$ was observed. We include this domain to show that \pdRMs can be used to obtain comparable results to CRAs when encoding tasks representable in deterministic context-free languages (DCFLs).
        
        For this experiment, we have used the implementation and experimental setup from Bester et al.'s Github repository as of May 12$^{\text{th}}$ 2025.\footnote{\href{https://github.com/TristanBester/counting-reward-machines}{https://github.com/TristanBester/counting-reward-machines}} Additional documentation is also available online.\footnote{\href{https://crm-74a68705.mintlify.app}{https://crm-74a68705.mintlify.app}} We trained six agents. Two were trained with the CRA from the repository of \citeauthor{Bester//:24}, one using Q-learning and the other using CQL, a variant of the original QRM algorithm~\cite{Toro-jair22} that uses counterfactual experiences to improve sample efficiency, adapted by \citeauthor{Bester//:24} (\citeyear{Bester//:24}) for CRAs. We refer to the two CRA-based agents as the \emph{Q-CRA agent} and \emph{CQL-CRA agent} respectively.        
        For the \pdRM-based agents, we trained a $1$-policy and a policy agent with vanilla Q-learning, and a $1$-policy and a policy agent with \CpdRM Q-learning. We refer to these agents as the \emph{1-\pdRM agent}, \emph{\pdRM agent}, \emph{1-\CpdRM} agent, and \emph{\CpdRM agent} respectively, and write \emph{-\pdRM} and \emph{-\CpdRM agents} to refer to all agents trained with the respective variant of Q-learning. 
        
        The results of the experiment are shown in \cref{fig:letterenv-experiment}. As can be seen, the agents that achieve the best performance are the -\CpdRM agents, the 1-\pdRM agent, and the CQL-CRA agent. In contrast, the Q-CRA and the \pdRM agents manage to increase their accrued rewards only towards the end of the training episodes; however, they both cannot match the performance of the other agents (although, given enough time, they will eventually reach the same performance). For the agents trained with the \pdRM, the results suggest that the smaller state space of the 1-\pdRM agent allows it to  more easily learn to achieve the task compared to the \pdRM agent; moreover, the smaller state space allow the 1-\pdRM agent to converge as fast as the CQL-CRA agent, which is trained with counterfactual experiences. Note how the 1-\CpdRM and the \CpdRM agents converged faster than the rest, and how the latter agent achieved the task as opposed to the \pdRM agent which did not. This shows that \CpdRM can increase sample efficiency. The results suggest that \pdRMs can be used in place of CRAs when tasks can be encoded in DCFLs.
                   
	\subsection{\textsc{1-TreasureMaze}}

\begin{figure*}[tb]
    \includegraphics[width=\linewidth]{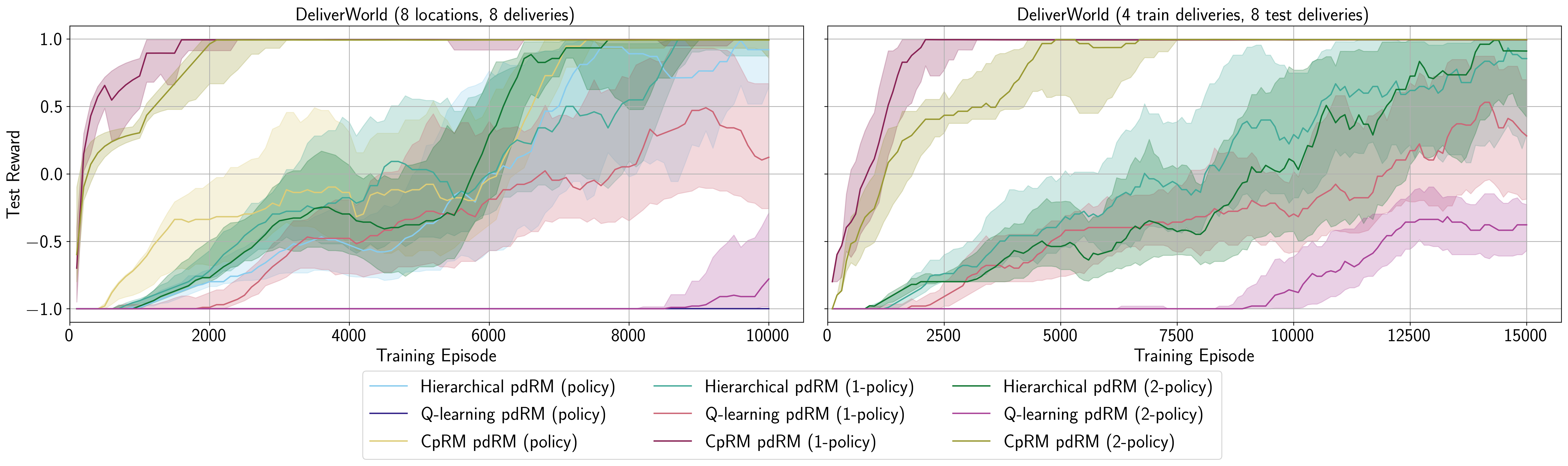}
    \caption{\textsc{DeliverWorld} results. Left plot: agents performed 8 deliveries during training and testing (\textsc{DeliverWorld-8}). Right plot: agents performed 4 deliveries during training episodes and 8 during testing episodes (\textsc{DeliverWorld-4-8}). For \textsc{DeliverWorld-8}, the only agents that consistently achieve the task at the end of training are the -\CpdRM agents, the Q-learning 1-\pdRM agent and the hierarchical 1-\pdRM and 2-\pdRM agents. The hierarchical \pdRM agent manages to achieve the task but not as consistently. The other Q-learning agents show the worse performance out of all agents, with the \pdRM agent flatlining at a reward of $-1$. On the other hand, for \textsc{DeliverWorld-4-8}, only the -\CpdRM agents and the hierarchical agents eventually learn to consistently achieve the task. This shows how \pdRMs can help in obtaining agents that complete longer tasks than the ones they were trained in.
    }
    \label{fig:deliverworld-experiment}
\end{figure*}

		In the \textsc{1-TreasureMaze} environment, the goal of the agent is to navigate a maze to find a treasure, retrieve it, and then return to the initial location by following the reversed path it took to find the treasure. For a detailed explanation of the environment, we refer the reader to the example in \cref{sec:pushdown-reward-machines}. The aim of this experiment is to show there are cases where \pdRMs are more sampe efficient than CRAs.

        For this experiment, we trained the same agents as in the \textsc{LetterEnv} experiment. The CRA encodes the task by encoding the path the agent is following to reach the treasure, similarly to the CM for the marked palindrome language in \cref{sec:theoretical-results}. At each step, the counters of the CRA are updated so that their configuration correctly encodes the path taken so far. As there are four possible directions identifying each step in the path, the number encoding a path is the translation in base $10$ of a number in base $4$. We define $0$, $1$, $2$, and $3$ to be respectively the directions $\mathsf{u}, \mathsf{d}, \mathsf{l}, $ and $\mathsf{r}$. Thus, at step $i$, we add to the current encoding of the path the value $4^i$ times the direction's value. In order to add this value, note how the CRA must repeatedly add $4$'s, as it can only add constant values in each of its transitions.
        
        Three different mazes were used: a 5$\times$5 maze, a 10$\times$10 maze, and a 20$\times$20. 
        In this experiment, we end episodes when either the maximum number of timesteps has elapsed, or a limit on wall clock time has been reached. In the left plot in \cref{fig:maze-experiment}, we show the results obtained on all mazes for the \textsc{1-TreasureMaze} experiment. We identify each agent with a colour, and each maze with a line style. In \Cref{sec:Appendix-E} of the Appendix, we provide further details on the maximum timesteps and wall clock limit per episode, and three further plots, each containing the results of the agents on each of the three mazes.

        As can be seen (left plot in \cref{fig:maze-experiment}), all the -\pdRM agents managed to achieve the task in the $5\times5$ maze. However, only the 1-\pdRM agent managed to learn to also achieve the task in the $10 \times 10$ and $20 \times 20$ mazes. We believe this is due to the fact that the state space for the 1-\pdRM agent is much smaller than that for the \pdRM agent.  The -\CpdRM agents could not learn to achieve the task as the training episodes ended prematurely due to the wall clock time limit. This is because of the number of counterfactual updates, which in this scenario is large given the size of the stack language. Note that it can be shown using the approach in Section \ref{sec:when-can-we-learn} that optimal 1-policies learnt by the 1-\pdRM agent have the same state values as optimal \pdRM and -\CpdRM policies.
        
        The Q-CRA agent learned to achieve the task in the 5$\times$5 maze, but not in the larger mazes. We believe the agent could not learn to achieve the task in the larger mazes because the counter values incur an exponential blowup with respect to the number of elapsed timesteps, resulting in an exponential number of operations per timestep and the episodes to end prematurely due to the wall clock time limit. Finally, the CQL-CRA agent never learnt to achieve the task. In addition to having the same issue with respect to the exponential blowup in the number of CRA operations that we observed in the Q-CRA agent, the CQL-CRA agent's training was further slowed down by the counterfactual policy updates. There are exponentially many counterfactual policy updates as the number of possible combinations of counter values is exponential in the maximum number of timesteps (due to an argument similar to the one we have given to establish \cref{thm:cra-space-blowup-omega}).

        In summary, the experiment shows that there are scenarios where \pdRMs are more suited to encode tasks and train agents than CRAs. Note that we do not claim this to always be the case, and believe that it should be decided on a case-by-case basis which of the two machines is more appropriate.
        
	\subsection{\textsc{MultipleTreasureMaze}}
        
		The \textsc{MultipleTreasureMaze} environment is a variant of the previous environment where the agent has to retrieve multiple treasures. 
        Before retrieving the treasures, the agent has to first find an intermediate ``safe'' location. After finding a safe location, the agent starts looking for treasures: as soon as it finds one, it must return to the safe location by following the reversed path it took to find such treasure. This is repeated until the agent finds all treasures, after which the agent observes a special event informing it that it has found all treasures. The agent must then return from the safe location to the initial exit location, by following the reversed path it took to reach the safe location. We include this experiment to show how \pdRMs can enable learning of more complex tasks.

        In this experiment we have trained only the -\pdRM agents. We have excluded the -\CpdRM and CRA-based agents because the wall clock time limit would have stopped the training episodes prematurely as in the \textsc{1-TreasureMaze} environment. Only a 10$\times$10 maze and a 20$\times$20 one were used. Note that, for this experiment, the maximum number of timesteps per episode for each maze is larger than in the \textsc{1-TreasureMaze} experiments.
        
        As can be seen (right plot in \cref{fig:maze-experiment}), the 1-\pdRM agent managed to eventually achieve the task in both mazes. Given the longer episodes compared to the \textsc{1-TreasureMaze} domain, the agents were also able to learn to consistently achieve the task in a relatively small number of episodes. This shows that a \pdRM can also be used in tasks where more complex operations with the stack are required. On the other hand, the \pdRM agent did not manage to learn to complete the task by the end of training in either maze. We believe that this is mainly due to the size of the policy, which seems to be too large for the agent to learn in this more complex task.

        The results show that, thanks to the flexible amount of stack information that \pdRMs allow in defining the agent's policy, \pdRMs allow us to train agents in more complex tasks.

    \subsection{\textsc{DeliverWorld}}
\begin{figure*}[tb]
        \includegraphics[width=\linewidth]{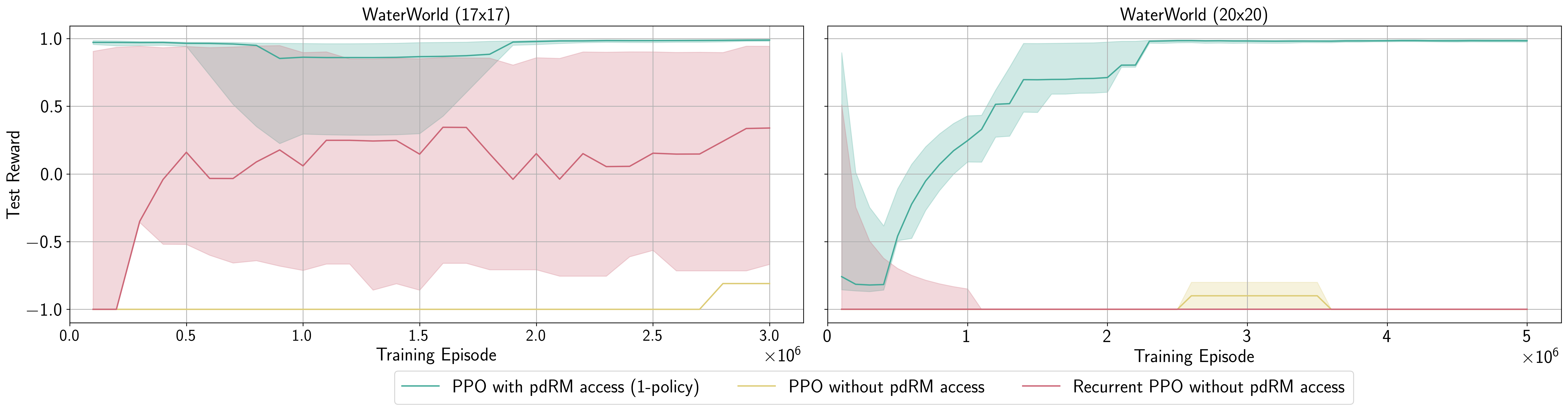}
        \caption{\textsc{WaterWorld} results, 17$\times$17 map (left) and 20$\times$20 map (right). We compare the performance of a 1-\pdRM agent trained with PPO against that of an agent trained with PPO and one of an agent trained with recurrent PPO. In the 17$\times$17 map, the 1-\pdRM agent is able to achieve the task very quickly. Only the recurrent PPO agent manages to considerably improve its performance, but does not match that of the 1-\pdRM agent. Similarly to the results of the 17$\times$17 map, in the 20$\times$20 map the 1-\pdRM agent is the only one that consistently achieves the task; significanly outperforming the other agents which do not improve their performance. These two plots show how the \pdRM is crucial in training agents to achieve this task.
        }
        \label{fig:waterworld-experiment}
\end{figure*}

        In \textsc{DeliverWorld}, the agent that is supposed to deliver packages to locations. Each location is assigned a ``type'' (e.g., shopping mall, clothing store, etc.), and there can be multiple locations of the same type. At the start of each episode, the agent observes an event which determines the sequence of delivery location types it needs to visit to achieve the task. This sequence is chosen randomly from a set. The aim of this experiment is to illustrate that there are scenarios where it is beneficial for the agent to have access to more than the top symbol on the stack. In this experiment, we show how \CpdRM and a hierarchical approach can improve sample efficiency. Moreover, in a variant of the experiment, we show how \pdRMs allow training of agents that can perform a larger number of subtasks in testing than in training episodes.
        
        In this experiment we have compared three sets of agents, all trained with access to the same \pdRM: one was trained with vanilla Q-learning, one with \CpdRM Q-learning, and the last with a hierarchical approach. 

        For the hierarchical approach, we have adapted the hierarchical algorithm proposed in \cite{Toro-jair22}. In our \pdRM-based version, the meta-policy has access to the current MDP state, \pdRM state and the stack of the \pdRM. The options' policies have access to the current MDP state, \pdRM state, and only the topmost symbol on the \pdRM stack. 
        
        We have used a 20$\times$20 grid, with two possible setups for the number of deliveries during training and the number of deliveries during testing. In the first, which we call \textsc{DeliverWorld-8}, the agents perform 8 deliveries during training and testing episodes. In the second, which we dub \textsc{DeliverWorld-4-8}, the agents perform 4 deliveries during training episodes and 8 during testing episodes.
        In \textsc{DeliverWorld-8}, we trained three agents per approach: a 1-policy agent, a 2-policy agent (denoted \emph{2-\pdRM} agent and \emph{2-\CpdRM} agent for the Q-learning approaches), and a policy agent.
        In \textsc{DeliverWorld-4-8}, we have trained two agents per approach: a 1-policy and a 2-policy agent. 

        \Cref{fig:deliverworld-experiment} shows the plots for both experiments. In both plots, we can clearly see that the most sample efficient agents are the -\CpdRM ones, followed by the hierarchical ones, and with the vanilla Q-learning ones coming last.

        In both scenarios, the -\CpdRM agents clearly outperform all other agents in speed of convergence. Notice how \CpdRM and the hierarchical approach improve the performance compared to the \pdRM agent in the \textsc{DeliverWorld-8} task: when trained with vanilla Q-learning, the \pdRM agent never learns to achieve the task; however the \CpdRM and hierarchical agent both do. This shows how the hierarchical approach and, when the stack language is not too large, \CpdRM can greatly increase sample efficiency.
            
        In the hierarchical approach, we can see that in the plot of \textsc{DeliverWorld-8} (left in \cref{fig:deliverworld-experiment}) the 2-\pdRM converges faster. We believe that this is due to the fact that, by having access to the top two symbols, the agent learns to correctly plan which delivery location type to visit first. On the other hand, in the plot of \textsc{DeliverWorld-4-8} (right in \cref{fig:deliverworld-experiment}), the 1-\pdRM and the 2-\pdRM seem to converge at the same pace. For this setup, we think that the fact that the agent needs to perform fewer deliveries during training penalizes agents with larger policy state spaces. Shorter training episodes (compared to the testing ones) imply that the agents can explore fewer states during training, thus leading to worse performance for the 2-\pdRM agent compared to the results it obtained in \textsc{DeliverWorld-8}. 

        Finally, for the Q-learning agents, the 1-\pdRM agent has the best performance in both plots. In principle, the other two Q-learning agents should be able to learn a policy that is at least as good as that of the 1-\pdRM agent. Note that of the Q-learning agents, the only one which is guaranteed to achieve an optimal policy in the limit is the \pdRM agent (see \cref{sec:when-can-we-learn}). However, given the state-space complexity of the policies and the limited number of episodes, the 2-\pdRM and \pdRM agents do not learn to achieve the task. 

        Interestingly, from the plot of \textsc{DeliverWorld-4-8}, we can see that agents trained with \pdRMs learn policies that, even though trained in smaller instances, performed adequately in larger testing instances. The -\CpdRM agents are able to achieve the test task relatively quickly, whereas the hierarchical agents eventually manage to achieve the test task by the end of training. Moreover, although the Q-learning agents do not achieve the test task, their testing rewards increase over time. We conjecture that with a longer training they should also converge to policies that achieve the test task. This suggests that \pdRMs can, to a certain degree, help agents to learn policies that can achieve longer tasks than the ones that they are trained in. 
        
    \subsection{\textsc{WaterWorld}}
        The last domain is based on the \textsc{WaterWorld} domain \cite{Karpathy-waterworld//:15}. The task consists of two steps. In the first, the agent must touch 8 balls. Each of these balls is labelled with a (unique) number from 2 to 9: whenever the agent touches the ball identified with $i$, the \pdRM pushes the parity of $i$ to the stack and the ball disappears. Once all 8 balls are touched, the \pdRM moves to a new state and the second phase of the task begins. In this phase, the agent has to touch one of two other balls, numbered  0 and 1 determined by the topmost parity on the stack. When the agent correctly touches the ball, it is moved to a new random location and the topmost symbol of the stack is popped. The task is considered achieved when the stack has been emptied. 

        The aim of this experiment is to evaluate whether \pdRMs can provide an advantage in continuous domains when using a deep algorithm. We trained three agents using Proximal Policy Optimization (PPO) \cite{Schulman//:17}: the first has access to the \pdRM that encodes the task (and is trained with a 1-policy), the second does not, and the third also does not have access to the \pdRM, but is trained with \emph{recurrent} PPO, \ie an implementation of PPO where both the actor and critic networks employ \emph{long short-term memory} (LSTM) units \cite{Hochreiter//:97}. LSTMs are a type of recurrent neural network, meaning they can capture sequential dependencies in data. Intuitively, using an LSTM should help the agent ``remember'' important events that have happened during the episode. Note that, as PPO is an on-policy algorithm, \CpdRM cannot be used. Instead, we used Soft Actor Critic (SAC) in conjunction with \CpdRM; however, the training of the agents did not terminate, as the server we used to run the \textsc{WaterWorld} experiments had a time limit on each job. We therefore do not report the results of agents trained with the counterfactual approach for this experiment.
        
        We have trained the agents in two different maps, a 17$\times$17 one and a 20$\times$20 one, and the results are shown in the left and right plot of \cref{fig:waterworld-experiment} respectively.
        
        As can be seen, in the 17$\times$17 map (left in \cref{fig:waterworld-experiment}) the agent trained with the \pdRM was able to achieve the task very quickly, and while it showed a small decrease in performance between the 5$^{\text{th}}$ and 20$^{\text{th}}$ episodes, it recovered and managed to outperform its initial results. The agent trained with recurrent PPO managed to improve its policy, but by the end of training its performance still showed high variance. Finally, the vanilla PPO agent only showed an increase in performance at the end of training, however, its resulting performance was not comparable to that of the 1-\pdRM agent or the recurrent PPO agent.
        In the 20$\times$20 map (right in \cref{fig:waterworld-experiment}) neither the recurrent PPO nor vanilla PPO agents learnt a policy that achieved the task in the test episodes. However, the agent with access to the \pdRM could learn a policy that consistently achieved the task in the test episodes from the 20$^{\text{th}}$ training episode. The results suggest that, in this scenario, using a \pdRM is crucial in training and in increasing sample efficiency, performing better than deep algorithms employing recurrent networks.  

\section{Related Work}
\label{sec:related-works}
The literature on reward machines is now quite large.
We focus specifically on work in which reward machines are modified in a way that is similar to our approach. The approach of \citet{Bester//:24} is closest to ours and is discussed in \cref{sec:theoretical-results}. The \emph{task monitors} of \citet{JothimuruganNeurIPS2019composable} are automata with numeric registers (and so are similar to \citeauthor{Bester//:24}'s counting reward automata), but the purpose of the registers was not to define more complicated temporal patterns but to keep track of the quantitative \emph{degrees} to which subtasks had been completed and constraints satisfied (so as to provide shaped rewards). \citet{Furelos//:23} augment reward machines by introducing a \emph{hierarchy}. In such hierarchies, RMs are able to call other RMs during execution. However, as they assume that there is always a \emph{leaf} RM that cannot call other RMs, and that each RM cannot be called by itself, even via recursive calls, they cannot express all tasks representable in DCFLs.
Another modification of reward machines, \emph{First-Order Reward Machines} (FORMs), is proposed by \citet{Ardon//:25}, where transitions are labelled by first-order logic formulas. However, FORMs only increase the expressivity of the events labelling transitions in RMs. On the other hand, our increased expressivity lies in the set of tasks that can be encoded by \pdRMs, which is strictly larger than the set of tasks encodable by RMs.

Other work has proposed approaches where RL agents are trained to achieve tasks that can be represented as deterministic context-free languages. In particular, \citeauthor{Hahn//:22} (\citeyear{Hahn//:22}) introduce recursive reinforcement learning, where \emph{recursive} MDPs (RMDPs) model the environment in which agents act. RMDPs generalise MDPs in that they consist of a set of MDPs where each MDP can ``call'' other MDPs. 
By keeping a stack of calls, RMDPs can encode tasks representable as DCFLs. Indeed, the authors specifically mention \emph{context-free reward machines} as a possible application of recursive RL; however they do not provide a formal argument. While it would be interesting to formally connect recursive RL to \pdRMs, this lies outside of the scope of this work and we leave it to future research.

\section{Conclusions}
\label{sec:conclusions}
In this paper, we have presented \emph{pushdown reward machines}, an extension of reward machines which can encode non-Markovian tasks representable as deterministic context-free languages. Compared to reward machines, \pdRMs are thus able to encode a strictly larger set of tasks. We have proposed two policy types for \pdRMs, one where the agent has access to the full stack (\emph{policies}), and one where it can access only its top $k$ symbols (\emph{$k$-policies}). In general, the state values of an optimal $k$-policy might not be as high as those of a policy. We described a procedure to check whether an MDP and a \pdRM are such that the two policy types have the same optimal state values. We have also compared \pdRMs to counting reward automata \cite{Bester//:24}, another extension of reward machines capable of encoding tasks representable as any recursively enumerable language. We showed that, when an agent trained with a \pdRM has access only to the top symbol of the \pdRM's stack, the size of a 1-policy can be  exponentially smaller (with respect to the episode's maximum length) than the size of a policy an agent trained with a CRA for the same task. Finally, in the experimental evaluation, we have shown how \pdRMs can be used in practice. We have seen how in certain scenarios it is more convenient to use \pdRMs than CRAs. We have also provided counterfactual and hierarchical algorithms specifically tailored for \pdRMs, and saw how they can increase convergence speed. Finally, we have used \pdRMs in a continuous domain, and showed how they can outperform state-of-the-art algorithms employing recurrent neural networks.

There are several directions for future work. 
First, we plan to investigate the mixed performance of \CpdRM in the \textsc{1-TreasureMaze} and \textsc{WaterWorld} tasks.
It would also be interesting to automatically synthesise \pdRMs, as was done with reward machines via, e.g., logic formalisms \cite{Camacho//:19,Varricchione//:23}, search and learning \cite{Toro-neurips19,Toro-aij23,Furelos//:23,HasanbeigJAMK24} or planning \cite{Illanes//:19,Varricchione//:24}. As deterministic pushdown automata are the underlying structure of \pdRMs, we think LR($k$) grammars \cite{Knuth//:65} could be good candidates to synthesise \pdRMs, as they can be easily translated into DPDAs \cite{Aho//:73}. 
Finally, as the stack can provide the agent with even further memory, \pdRMs can be an interesting alternative approach to RMs \cite{Toro-neurips19,Toro-aij23} in dealing with partially observable environments.

\section*{Acknowledgements}
\label{sec:acknowledgements}
We thank the anonymous reviewers for their helpful comments. The second and final authors gratefully acknowledge funding from the Natural Sciences and
Engineering Research Council of Canada (NSERC) and the Canada CIFAR AI Chairs
Program. Resources used in preparing this research
were provided, in part, by the Province of Ontario, the Government of
Canada through CIFAR, and companies sponsoring the Vector Institute for
Artificial Intelligence (www.vectorinstitute.ai/partners). Finally, the second and final authors
thank the Schwartz Reisman Institute for Technology and Society for
providing a rich multi-disciplinary research environment.

\bibliographystyle{kr}
\bibliography{biblio.bib}

\clearpage

\section*{Appendix}
\label{sec:appendix}
\renewcommand{\thesection}{\Alph{section}}
\setcounter{section}{0}
\renewcommand*{\theHsection}{chX.\the\value{section}}

\section{Example Translation From CRA to pdRM}\label{sec:Appendix-A}
    In this section of the Appendix, we provide an example of a translation from a 1-counter CRA to a pdRM. Specifically, we consider the CRA given as an example in the work by Bester et al. for the \textsc{LetterEnv} task. The CRA is shown in \cref{fig:letterenv-CRA}.
    
    \begin{figure}[H]
        
        \begin{tikzpicture}[shorten >=1pt,node distance=2cm,auto]

          \node[state,initial]      (u_0)                {$u_0$};
          \node[state]              (u_1) [right=of u_0] {$u_1$};
          \node[state, minimum size=0.3em, fill=black]    (u_2) [below=1cm of u_0] {};
          \node[state, minimum size=0.3em, fill=black]    (u_3) [below=1cm of u_1] {};
        
          \path[->] (u_0) edge              node        {{\scriptsize $\langle P_B, [1], [0], 0\rangle$}} (u_1)
                          edge [loop above,text width=3cm,align=center] node        {{\scriptsize $\langle P_A, [1], [1], 0\rangle$\\[.2cm]$\langle P_A, [0], [1], 0\rangle$}} ()
                          edge              node [left,text width=3cm,align=center] {{\scriptsize $\langle \neg P_A \land \neg P_B, [0], [1], 0 \rangle$\\[.2cm]$\langle \neg P_A \land \neg P_B, [1], [0], 0\rangle$}} (u_2)
                    (u_1) edge              node [right,text width=3cm,align=center] {{\scriptsize $\langle P_B, [1], [0], 0 \rangle$\\[.2cm]$\langle \tau, [0], [0], 1 \rangle$}} (u_3)
                          edge [loop above] node        {{\scriptsize $\langle P_C, [1], [-1], 0\rangle$}} ();
        \end{tikzpicture}%
        \caption{CRA for the \textsc{LetterEnv} task. The states drawn as black circles are the terminal states.}
    \label{fig:letterenv-CRA}
    \end{figure}

    An equivalent pdRM can be seen in \cref{fig:letterenv-equivalent-pdrm}. In it, we use the special symbol $\#$ to denote the end of the stack (we set the initial stack symbol $Z$ to be $\#$). The special symbol $\#$ is used to simulate the zero-test on the counter's value done in the CRA of \cref{fig:letterenv-CRA}: if the top symbol is $\#$, then the counter's value would be zero; otherwise, if the top symbol is not $\#$, then the counter's value would not be zero.

    \begin{figure}[H]
        
        \begin{tikzpicture}[shorten >=1pt,node distance=2cm,auto]

          \node[state,initial]      (u_0)                {$u_0$};
          \node[state]              (u_1) [right=of u_0] {$u_1$};
          \node[state, minimum size=0.3em, fill=black]    (u_2) [below=1cm of u_0] {};
          \node[state, minimum size=0.3em, fill=black]    (u_3) [below=1cm of u_1] {};
        
          \path[->] (u_0) edge              node        {{\scriptsize $\{ P_B\}, A/A, 0$}} (u_1)
                          edge [loop above,text width=3cm,align=center] node        {{\scriptsize $ \{P_A\}, A / AA, 0$\\[.2cm]$\{ P_A \}, \#/A\#, 0$}} ()
                          edge              node [left,text width=3cm,align=center] {{\scriptsize $\{ \neg P_A, \neg P_B\}, \#/A, 0 $\\[.2cm]$\{ \neg P_A, \neg P_B\}, A/A, 0$}} (u_2)
                    (u_1) edge              node [right,text width=3cm,align=center] {{\scriptsize $\{ P_B \}, A/A, 0 $\\[.2cm]$\{\tau\}, \#/\#, 1$}} (u_3)
                          edge [loop above] node        {{\scriptsize $\{P_C\}, A/\epsilon, 0$}} ();
        \end{tikzpicture}%
        \caption{pdRM for the \textsc{LetterEnv} task equivalent to the CRA in \cref{fig:letterenv-CRA}.}
    \label{fig:letterenv-equivalent-pdrm}
    \end{figure}

\section{General Translation From 1-Counter CRA to pdRM}\label{sec:Appendix-B}

    In \Cref{sec:Appendix-A} of the Appendix we have seen an example of a translation from a CRA to a pdRM. In general, it is not possible to translate an arbitrary CRA to a pdRM: as CRA can have multiple counters, they can be strictly more expressive than pdRMs. However, for CRAs with one counter\footnote{It should be noted that here by ``CRA'' we are referring specifically to the so-called ``\emph{constant} CRAs'', \ie the CRAs where the output function outputs real values (rewards) and not functions. As shown by Bester et al., the two formalisms are equally expressive.} there exists always a pdRM which encodes the same non-Markovian reward function. In this section of the Appendix, we provide a formal construction that, given a 1-counter CRA, produces a pdRM that encodes the same non-Markovian reward function. 
    
    Fix a given 1-counter CRA $\mathcal{C} = \langle U, F, \Sigma, \Delta, \delta, \lambda, u_0\rangle$. From it, we build the corresponding pdRM $\RM^\mathcal{C} = \la \RMStates^\mathcal{C}, \RMInState^\mathcal{C}, \RMFinStates^\mathcal{C}, \RMAlphabet^\mathcal{C}, \PRMStackAlphabet^\mathcal{C}, \PRMStackInSymbol^\mathcal{C}, \RMTransition^\mathcal{C}, \RMReward^\mathcal{C} \ra $ as follows:

    \smallskip

    \begin{itemize}
        \item $U^\mathcal{C} := U \cup U^\delta$, where $U^\delta$ is a set of ``helper'' states used to correctly simulate decreases in the counter's value. Following, when defining the transition function $\RMTransition^\mathcal{C}$, we will properly define the set $U^\delta$;

        \item $\RMInState^\mathcal{C} := u_0$;

        \item $\RMFinStates^\mathcal{C} := \RMFinStates$;

        \item $\RMAlphabet^\mathcal{C} := \Sigma \cup \{ \epsilon \}$, where $\epsilon$ is a  symbol not in $\Sigma$;

        \item $\PRMStackAlphabet^\mathcal{C} := \{ \#, \Box\}$. The $\#$ symbol is used to denote the end of the stack and, similarly to the pdRM from \cref{sec:Appendix-A} of the Appendix, will be used to simulate the zero-test of the counter's value. The $\Box$ symbol will instead be used to represent a single unit on the counter: whenever the counter's value in the CRA is modified by adding/subtracting $n$, then the stack will be changed by pushing/popping $n$ $\Box$ symbols;

        \item $\PRMStackInSymbol^\mathcal{C} := \#$.

    \end{itemize}

    \smallskip

    We now define the transition and reward functions $\RMTransition^\mathcal{C}, \RMReward^\mathcal{C}$ of the pdRM. In order to define $\RMTransition^\mathcal{C}$, consider an arbitrary CRA configuration $\langle q, c \rangle$, where $q$ is the CRA state and $c$ the value of the CRA counter. Upon observing the input symbol $\sigma$, the CRA transitions to the configuration $ \langle q', c' \rangle = \delta(q, \sigma, \mathbb{1}(c))$ and outputs reward $ r = \lambda(q, \sigma, \mathbb{1}(c))$. Here, $u, u' \in U$ are states of the CRA, $c$ and $c'$ are the value of the CRA's counter, $\sigma \in \Sigma$ is the input symbol, and $\mathbb{1}(c)$ is the output of the zero-test function, defined as follows:
    \[
        \mathbb{1}(c) := \begin{cases}
            0 & \text{if } c = 0\\
            1 & \text{otherwise}
        \end{cases}
    \]

    As we have mentioned before, the zero-test function can be easily simulated in the pdRM by checking whether the top symbol of the stack is $\#$ or not. Hence, to define $\RMTransition^\mathcal{C}$ and $ \RMReward^\mathcal{C}$, if $\mathbb{1}(c) = 0$ we require that the top stack symbol in the input of the two function is $\#$, and otherwise that it is $\Box$. Following, when defining $\RMTransition^\mathcal{C}$ and $\RMReward^\mathcal{C}$, we let $\gamma$ be the top symbol as we have just specified.
    
    Note that $c'$ is obtained by adding or subtracting to $c$ some value $m$. To define $\RMTransition^\mathcal{C}$ and $ \RMReward^\mathcal{C}$, we differentiate between the case where $m \geq 0$ and the one where $m < 0$:

    \smallskip

    \begin{itemize}
        \item $m \geq 0$: in this case, the pdRM stack needs to be modified by pushing $m$ $\Box$ symbols. This can be easily done in a pdRM, by pushing the symbols in a single transition. For the transition function, we define $\RMTransition^\mathcal{C}(q, \sigma, \gamma) = (q', \Box^m \gamma)$, where $\Box^m$ is a string of $m$ $\Box$ symbols. For the reward function, we define $\RMReward^\mathcal{C}(q, \sigma, \gamma) = r$. 

        \item $m < 0$: in this case, the pdRM stack needs to be modified by popping $m$ $\Box$ symbols. Unlike the previous case, in this one we cannot pop all the $m$ symbols in a single transition, as pdRM allow to only pop one symbol (the topmost on the stack) per transition. We therefore use ``helper'' states to correctly pop the $m$ symbols. Let $u_{q, \sigma, \gamma}^1, \ldots, u_{q, \sigma, \gamma}^{|m|} \in U^\delta$ be the helper states for this transition. Starting from state $q$, the transition function will first move to the first helper state upon observing $\sigma$, and it will then traverse all the helper states until reaching the state $q'$ via $\epsilon$-transitions. While transitioning through the helper states, the transition function will also pop the topmost symbol if it is a $\Box$ (if it the $\#$, it will not pop it). In the meanwhile, the reward function will always output $0$, until the transition from the last helper state to the state $q'$ when it will output the reward $r$. 
        
        Thus, we define the transition and reward functions as follows:
        \begin{itemize}
            \item $\RMTransition^\mathcal{C}$: first, transition from state $q$ to state $u_{q, \sigma, \gamma}^1$ without popping the top symbol, \ie $\RMTransition^\mathcal{C}(q, \sigma, \gamma) = (u_{q, \sigma, \gamma}^1, \gamma)$. We now differentiate two cases: if $|m| = 1$, then it means that there are no more helper states to traverse and the pdRM should now move to the new state $q'$. We thus define two transitions from $u_{q, \sigma, \gamma}^1$, depending on whether the top symbol $\gamma$ is $\Box$ or $\#$. If $\gamma = \Box$, we should pop the symbol as the value of the CRA's counter was decreased by $1$; if $\gamma = \#$ then we do not pop the symbol as it means that the value of the CRA's counter was already $0$ (recall that the value of a CRA counter is always non-negative). Thus, we let $\RMTransition^\mathcal{C}(u_{q, \sigma, \gamma}^1, \epsilon, \Box) = (q', \epsilon)$ and $\RMTransition^\mathcal{C}(u_{q, \sigma, \gamma}^1, \epsilon, \#) = (q', \#)$. On the other hand, if $|m| > 1$, then there is at least one other helper state $u_{q, \sigma, \gamma}^2$. Analogously to the previous case, we define two transitions from $u_{q, \sigma, \gamma}^1$ to $u_{q, \sigma, \gamma}^2$, one for $\gamma = \#$ (where $\#$ is not popped) and one for $\gamma = \Box$ (where $\Box$ is popped). We now repeat the same construction for the rest of the helper states $u_{q, \sigma, \gamma}^3, \ldots, u_{q, \sigma, \gamma}^{|m|}$ (if there are any);

            \item $\RMReward^\mathcal{C}$: define $\RMReward^\mathcal{C}(q, \sigma, \gamma) = 0$ and $\RMReward^\mathcal{C}(u_{q, \sigma, \gamma}^i, \epsilon, \gamma) = 0$ for all $1 \leq i < |m|$ and $\gamma \in \{ \#, \Box \}$. Then, let $\RMReward^\mathcal{C}(u_{q, \sigma, \gamma}^{|m|}, \epsilon, \gamma) = r$ for $\gamma \in \{ \#, \Box \}$. 
        \end{itemize}
    \end{itemize}

    From the construction, it should be clear that both the CRA and the pdRM will output the same rewards while reading any input word.

\section{Optimality of Full-Policies}\label{sec:Appendix-C}
    In this section of the Appendix we present an environment and experimental results that show in practice how full-policies can be more optimal than top-$k$-policies, and how top-$k$-policies can be more optimal than top-$k'$-policies, for $k'<k$ (cfr. \Cref{sec:when-can-we-learn} in the paper). 

    For this experiment, we have devised an environment called ``\textsc{PaintWorld}'', in which the agent's goal is to request an adequate amount of soap to clean paint stains. We have encoded the task using a pdRM: whenever an episode starts, the stack is initialized by pushing a string of $n$ paint symbols, where $n$ is randomly sampled and ranges from $1$ to $5$. The actions the agent can perform consist in requesting $i$ units of soap, for $i$ ranging also from $1$ to $5$. Whenever the agent asks for $i$ units of soap, it receives a penalty of $\frac{i}{i + 1}$. Hence, if there are $i$ paint stains, it is always optimal to ask for precisely $i$ units of soap. The pdRM always outputs $0$ as its reward, 

    To show how agents that have access to more symbols on the stack can achieve better policies (with respect to the achieved reward), we have trained $5$ agents, where the $i$-th agent has access to the top $i$ symbols on the stack. Each agent was trained for a total of $500$ episodes, each lasting at most $5$ timesteps, and evaluated every $5$ training episodes. Agents were tested by running $5$ test episodes, where in the $i$-th test episode the agents had to clean $i$ paint stains. In the plot in \cref{fig:paintworld} we report the average rewards obtained by the agents in each test. As can be seen in the plot, the more symbols the agent can access on the stack, the better the policy it learns. The only agent who achieves the optimal performance is the full-policy agent: indeed, in the test episodes the full-policy agent is the only one which optimally chooses to request $i$ units of soap whenever it needs to clean $i$ stains of paint. 
 
    \begin{figure}
        \centering
        \includegraphics[width=\linewidth]{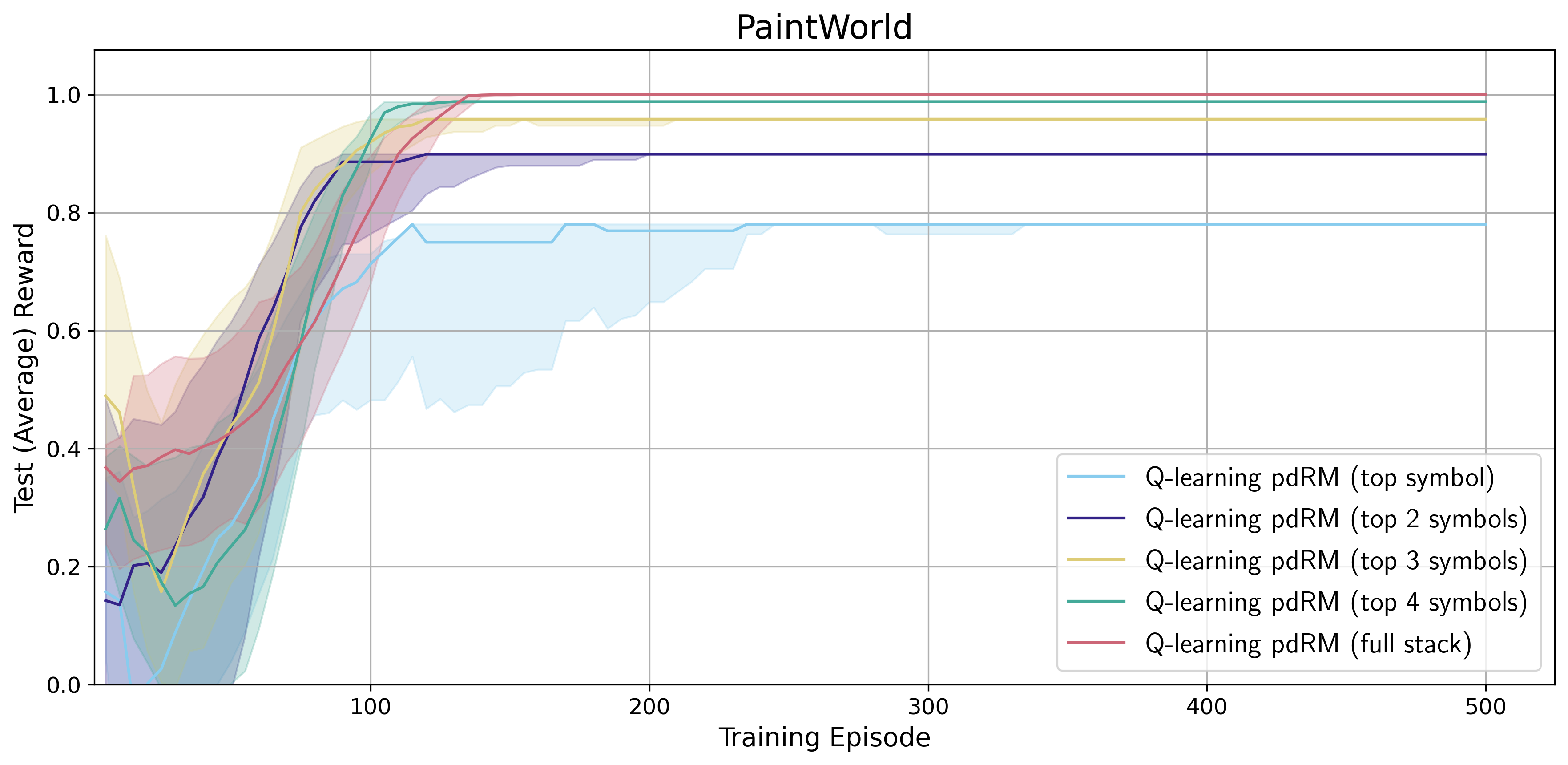}
        \caption{\textsc{PaintWorld} results. While rewards were normalized between $-1$ and $1$, here we only show the positive half of the $y$ axis to make the lines more distinguishable.}
        \label{fig:paintworld}
    \end{figure}

In this section, we provide, for each experiment, the experimental setup (e.g., environment size, episode length) and the hyperparameters settings. 

For each experiment, we report the results obtained by the agents across 10 different training runs. During training, after a set amount of training episodes (which varies across the different experiments), we evaluated the agents by running 10 test episodes. We randomly selected $10$ seeds $x_1, \ldots, x_{10}$, and for each agent and experiment, the $i$-th run was seeded with seed $x_i$. 

\section{Counterfactual Algorithm for \texorpdfstring{$k$}{k}-Policies}\label{sec:Appendix-D}
    Here we give a variant of the counterfactual algorithm presented in \cref{sec:counterfactual} which can be used to learn a $k$-policy. In the algorithm that follows, for a given pdRM stack string $\PRMStackString$, we denote with $\PRMStackString_{\upharpoonright k}$ the substring obtained by taking the top $k$ symbols of $\PRMStackString$. 

\begin{algorithm}[tb]
\caption{Q-learning with \CpdRM~(policy)} \label{alg:topk-cpdrm}
\label{alg:counterfactual-pdRM-topk-algorithm}
{\begin{flushleft}{\bfseries Input:} MDP-pdRM $\mMDPpdRM$, num\_episodes\end{flushleft}}
\begin{algorithmic}[1]
    \STATE Initialise $\qapprox(s, u , \zeta, a)$ arbitrarily for each $\MDPState \in \MDPStates, \RMState \in \RMStates, \PRMStackString \in \PRMStackAlphabet^k, a \in A$
    \STATE $\mathcal{O} \gets \{\}$ \COMMENT{Set of observed stack strings}
    \FOR{$\ell\leftarrow 0$ \TO num\_episodes} 
        \STATE $s \gets$ EnvInitialState(), $u \gets u_0$ and $\PRMStackString \gets \PRMStackInSymbol$
        \WHILE{$s$ is not terminal \AND $u \not\in F$} 
            \STATE $\mathcal{O} \gets \mathcal{O} \cup \{ \PRMStackString\}$
            \STATE Sample action $a$ using policy derived from $\qapprox$ (e.g., $\epsilon$-greedy) given current state $\tuple{s,\PRMStackString, u}$
            \STATE Take action $a$ and observe the next state $s'$
            \FOR{pdRM state $\RMState_c \in \RMStates$}
                \FOR{stack string $z_c\zeta_c \in \mathcal{O}$}
                    \STATE $\tuple{\RMState'_c, \PRMStackString'_c} \gets \RMTransition(\RMState_c, L(s, a, s'), \PRMStackSymbol_c)$
                    \STATE $r_c \gets \RMReward(u_c, L(s, a, s'), z_c)$
                    \IF{$s'$ is terminal \OR $u_c' \in F$} 
                        \STATE $\qapprox(s,u_c,(z_c\zeta_c)_{\upharpoonright k},a) \xleftarrow{\alpha} r_c$
                    \ELSE
                        \STATE $\qapprox(s,u_c,(z_c\zeta_c)_{\upharpoonright k},a) \xleftarrow{\alpha} r_c +\ $\\
                        $\qquad\qquad \gamma \max_{a'\in A}{\qapprox(s',u_c',(\zeta_c'\zeta_c)_{\upharpoonright k},a)}$
                \ENDIF
                \ENDFOR
            \ENDFOR
            \STATE Update $\langle \MDPState, \RMInState, \PRMStackString \rangle \gets 
            \tuple{\MDPState', \RMState', \PRMStackString'} $, where $\tuple{\RMState', \PRMStackString'} \dashv_{\Labelling(\MDPState, \MDPAction, \MDPState')} \tuple{\RMState, \PRMStackString}$
        \ENDWHILE
    \ENDFOR
\end{algorithmic}
\end{algorithm}

\section{Experimental Evaluation Details}\label{sec:Appendix-E}

In this section, we provide, for each experiment, the experimental setup (e.g., environment size, episode length) and the hyperparameters settings. 

For each experiment, we report the results obtained by the agents across 10 different training runs. During training, after a set amount of training episodes (which varies across the different experiments), we evaluated the agents by running 10 test episodes. We randomly selected $10$ seeds $x_1, \ldots, x_{10}$, and for each agent and experiment, the $i$-th run was seeded with seed $x_i$. 

\subsection*{\textsc{LetterEnv}}
    We note that for this experiment we have used the set of hyperparamters provided in the most updated version (as of May 7$^{th}$, 2025) of the counting reward automata's repository (\href{https://github.com/TristanBester/counting-reward-machines/tree/main}{https://github.com/TristanBester/counting-reward-machines/}). Following are the hyperparameters:

    \begin{itemize}
        \item Learning rate $\alpha = 0.01$;
        \item Exploration rate $\epsilon = 0.01$;
        \item Discount factor $\gamma = 0.99$.
    \end{itemize}

    As for the environment and experimental setup, we have trained the agents for a total of $5,000$ episodes, each lasting at most $300$ timesteps. Testing episodes were run every $100$ training episodes. The map's size was $3\times7$ (rows $\times$ columns).

    \Cref{fig:letterenv-pdrm} contains the \pdRM used in this experiment. 

    \begin{figure}[H]
        \centering
        \resizebox{\linewidth}{!}
        {
            \begin{tikzpicture}[shorten >=1pt,auto]
                \node[state,initial]      (u_0)                           {\Large$u_0$};
                \node[state]              (u_1) [right=5cm of u_0]        {\Large$u_1$};
                \node[state, accepting]   (u_2) [right=5cm of u_1]  {\Large$u_2$};
    
                \path[->]
                    (u_0)   
                        edge [loop below]   node [align=center] {$\{ A \}, \PRMStackSymbol / A\ \PRMStackSymbol, -0.01$\\$ \{C\}, \PRMStackSymbol / \PRMStackSymbol, -0.01$\\$ \{  \}, \PRMStackSymbol / \PRMStackSymbol, -0.01$}    ()
                        edge    node    {$\{ B \}, \PRMStackSymbol / \PRMStackSymbol, -0.01$}  (u_1)
    
                    (u_1)
                        edge [loop below]   node [align=center] {$ \{ A \}, \PRMStackSymbol / \PRMStackSymbol, -0.01$\\$\{B\}, \PRMStackSymbol / \PRMStackSymbol, -0.01$\\$ \{  C\}, \PRMStackSymbol / \epsilon, -0.01$\\$ \{  \}, \PRMStackSymbol / \PRMStackSymbol, -0.01$}    ()
                        edge node [align=center]    {$\{C\}, \# / \#, 1$}    (u_2);
            \end{tikzpicture}
        }
        \caption{Pushdown reward machine for the \textsc{LetterEnv} domain. $\#$ is a special symbol used to denote the start of the stack.
        }
        \label{fig:letterenv-pdrm}
    \end{figure}

\subsection*{\textsc{1-TreasureMaze}}

    First, we include in \cref{fig:maze-5x5} a rendition of the smallest maze that we have used in this experiment. 

     \begin{figure}[H]
            \centering
            \scalebox{.6}{
                \begin{tikzpicture}
                [
                    y=-1cm,
                    box/.style={rectangle,draw=black,thick, minimum size=1cm},
                ]
                    \draw[xstep=1,ystep=-1] (0,5) grid (5,0);

                    \node[box] at (1.5, 0.5){{\Huge\textsf{x}}};
                    \node[box] at (3.5, 3.5){{\Huge\textsf{t}}};
            
                    \node[box,fill=black] at (0.5, 0.5){};
                    \node[box,fill=black] at (2.5, 0.5){};
                    \node[box,fill=black] at (3.5, 0.5){};
                    \node[box,fill=black] at (4.5, 0.5){};
            
                    \node[box,fill=black] at (0.5, 1.5){};
                    \node[box,fill=black] at (0.5, 2.5){};
                    \node[box,fill=black] at (0.5, 3.5){};
            
                    \node[box,fill=black] at (4.5, 1.5){};
                    \node[box,fill=black] at (4.5, 2.5){};
                    \node[box,fill=black] at (4.5, 3.5){};
            
                    \node[box,fill=black] at (2.5, 2.5){};
                    \node[box,fill=black] at (3.5, 2.5){};
            
                    \node[box,fill=black] at (0.5, 4.5){};
                    \node[box,fill=black] at (1.5, 4.5){};
                    \node[box,fill=black] at (2.5, 4.5){};
                    \node[box,fill=black] at (3.5, 4.5){};
                    \node[box,fill=black] at (4.5, 4.5){};
                \end{tikzpicture}
                }
        \caption{Rendition of the 5$\times$5 maze. \textsf{x} is the initial and exit location, \textsf{t} the location of the treasure, and the black cells are walls.}
        \label{fig:maze-5x5}
        \end{figure}

    We have used the following set of hyperparameters for this experiment:
    \begin{itemize}
        \item Learning rate $\alpha = 0.5$;
        \item Exploration rate $\epsilon$ initialised at $1$, and linearly decayed to $0.01$ by multiplying it by $0.995$ after every episode;
        \item Discount factor $\gamma = 0.99$.
    \end{itemize}

    Agents were trained for a total of $10,000$ episodes and were tested every $100$ episodes. The maximum number of timesteps and limit for wall clock times per episode are as follows: 
    \begin{itemize}
        \item $5\times5$: $15$ timesteps, $0.03$ seconds per episode;
        \item $10\times10$: $15$ timesteps, $0.06$ seconds per episode;
        \item $20\times20$: $300$ timesteps, $0.1$ seconds per episode.
    \end{itemize}
    For the time limit, we have measured the wall clock time of the training process by using the Python \texttt{time.perf\_counter()} function. 

    \Cref{fig:1-treasuremaze-pdrm} contains the \pdRM used in this experiment.

    \begin{figure}[H]
        \centering
        \resizebox{\linewidth}{!}
        {
            \begin{tikzpicture}[shorten >=1pt,auto]
                \node[state,initial]      (u_0)                           {\Large $u_0$};
                \node[state]              (u_1) [right=5cm of u_0]        {\Large $u_1$};
                \node[state, accepting]   (u_3) [right=5cm of u_1]  {\Large $u_3$};
                \node[state, accepting]   (u_2) [above=3cm of u_3]  {\Large $u_2$};
    
                \path[->]
                    (u_0)   
                        edge [loop below]   node [align=center] {\Large $ \{ \mathsf{dir} \}, \PRMStackSymbol / \mathsf{dir}\ \PRMStackSymbol, 0$}    ()
                        edge    node    {\Large $ \{ \mathsf{dir}, \mathsf{t} \}, \PRMStackSymbol / \mathsf{dir}\ \PRMStackSymbol, 1$}  (u_1)
    
                    (u_1)
                        edge [loop below]   node [align=center] {\Large $\{ \mathsf{\overline{dir}} \}, \mathsf{dir} / \epsilon, 1$}    ()
                        edge    node    {\Large $ \{ \mathsf{\overline{dir}}, \mathsf{x} \}, \mathsf{dir} / \epsilon, 1\mathrm{e}{5}$}  (u_3)
                        edge [bend left]   node [align=center]    {\Large $\{ \mathsf{dir}' \}, \mathsf{dir} / \epsilon, -1\mathrm{e}{5}$\\\Large $ \{ \mathsf{dir}', \mathsf{x} \}, \mathsf{dir} / \epsilon, -1\mathrm{e}{5}$}    (u_2);
            \end{tikzpicture}
        }
        \caption{
        Pushdown reward machine for the \textsc{1-TreasureMaze} domain. For an explanation of the labels we refer the reader to Figure 1 in the main paper.
        }
        \label{fig:1-treasuremaze-pdrm}
    \end{figure}

    As we have used three different maps for this experiment, we include here (\cref{fig:1-treasuremaze-5x5-results,fig:1-treasuremaze-10x10-results,fig:1-treasuremaze-20x20-results}) also the results on each of these maps. 

    \begin{figure}
        \centering
        \includegraphics[width=\linewidth]{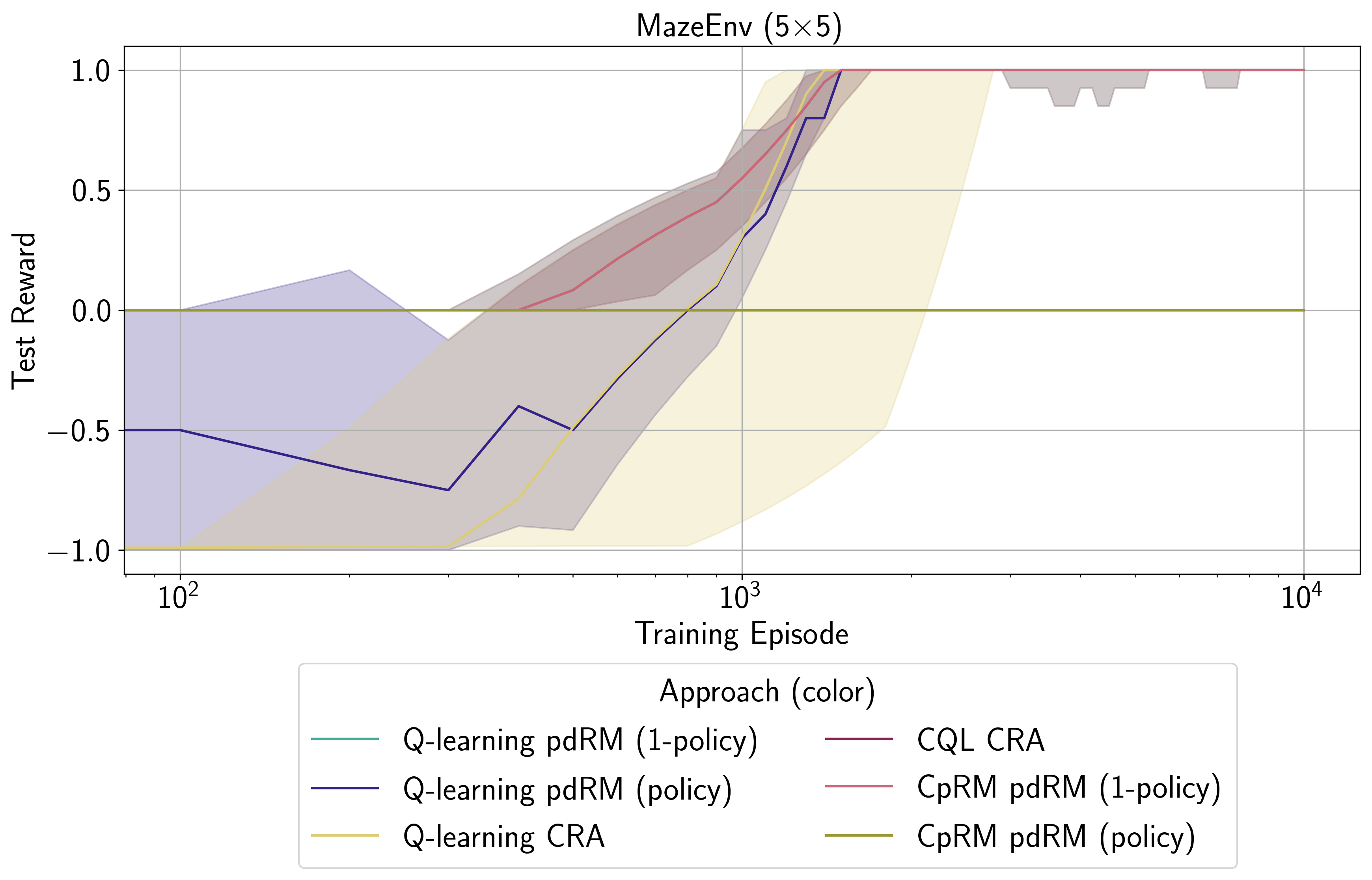}
        \caption{\textsc{1-TreasureMaze} results, $5\times5$ map.}
        \label{fig:1-treasuremaze-5x5-results}
    \end{figure}

    \begin{figure}
        \centering
        \includegraphics[width=\linewidth]{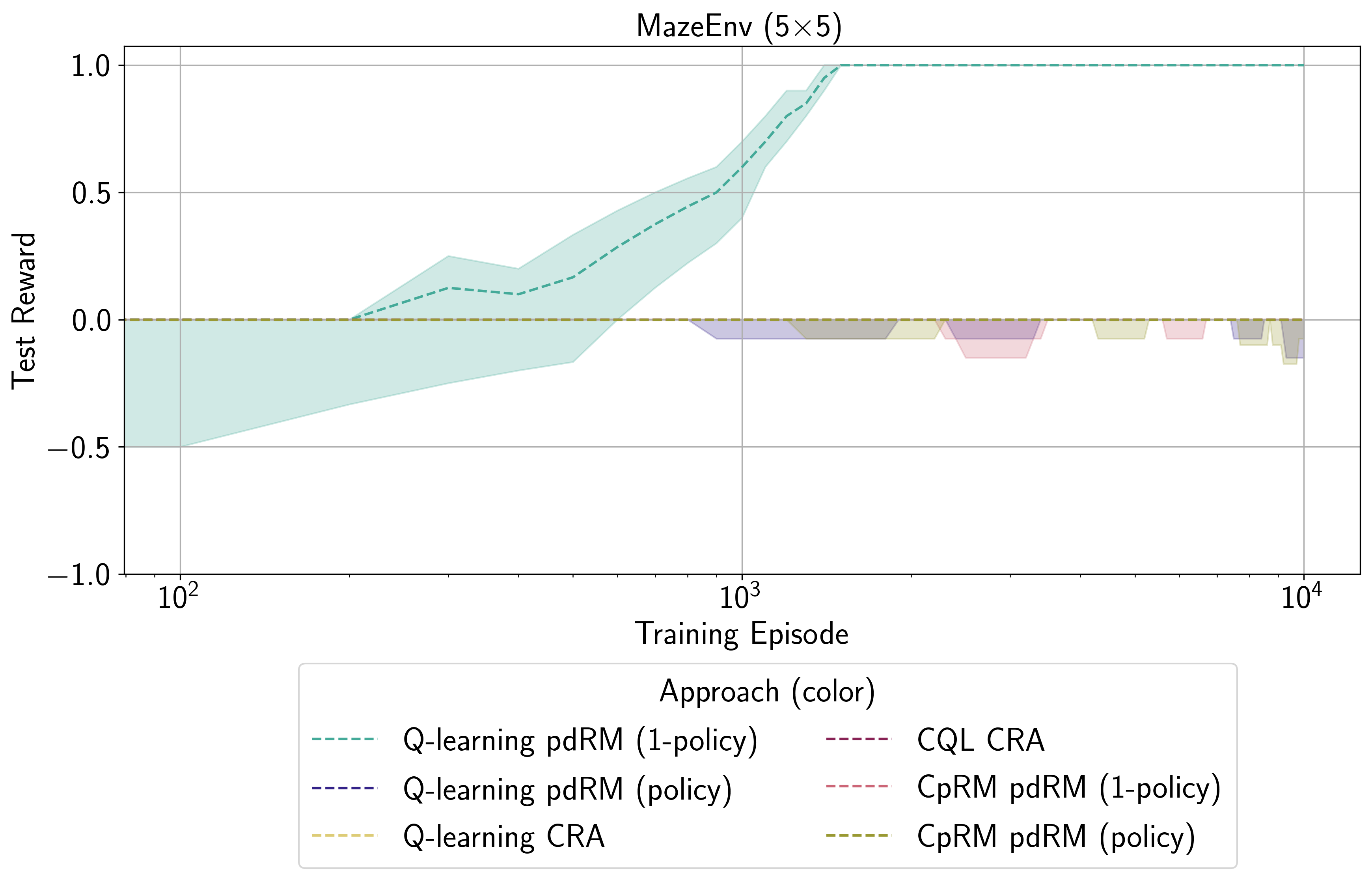}
        \caption{\textsc{1-TreasureMaze} results, $10\times10$ map.}
        \label{fig:1-treasuremaze-10x10-results}
    \end{figure}

    \begin{figure}
        \centering
        \includegraphics[width=\linewidth]{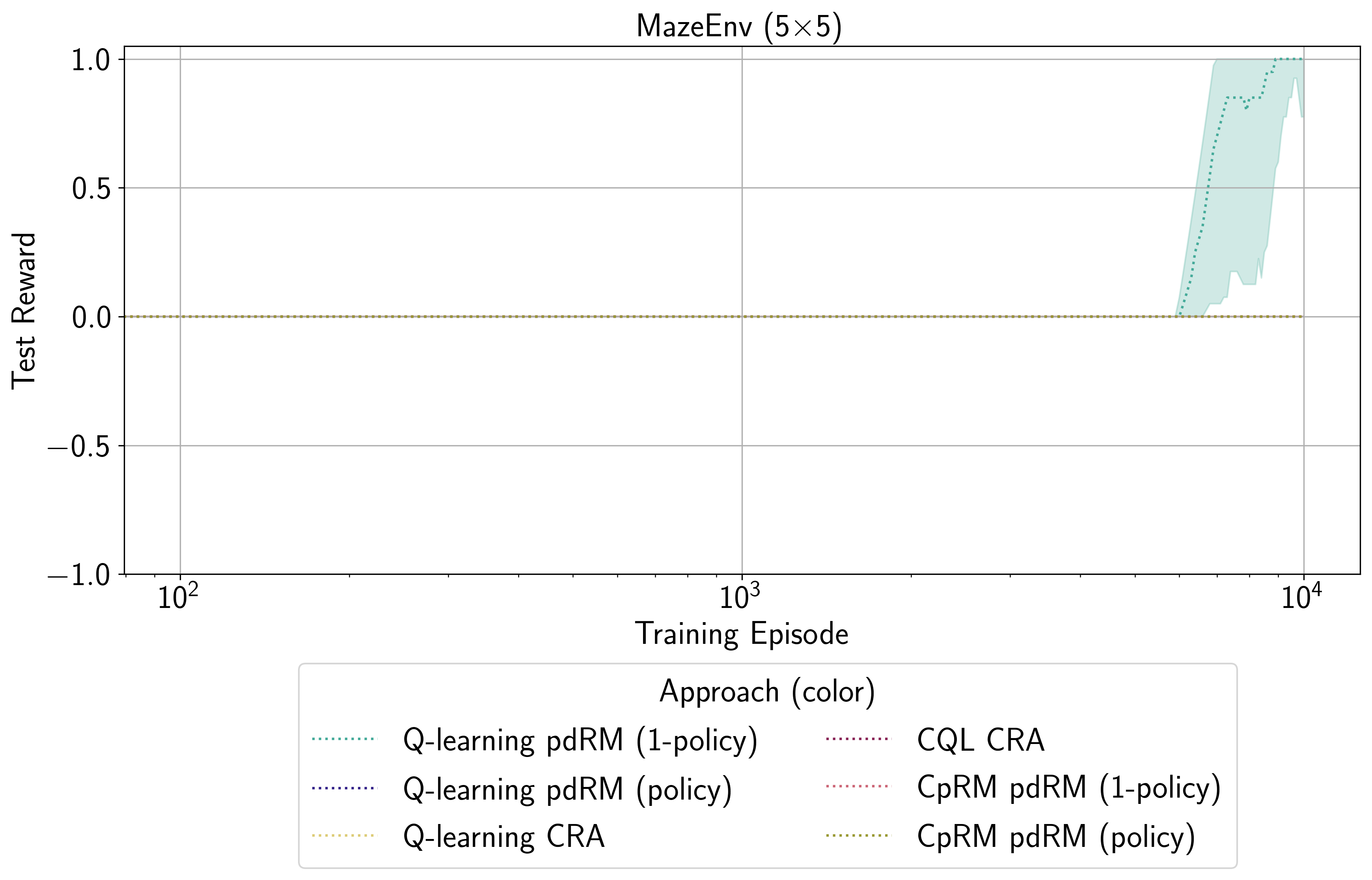}
        \caption{\textsc{1-TreasureMaze} results, $20\times20$ map.}
        \label{fig:1-treasuremaze-20x20-results}
    \end{figure}

\subsection*{\textsc{MultipleTreasureMaze}}
    We have used the following set of hyperparameters for this experiment:
    \begin{itemize}
        \item Learning rate $\alpha = 0.5$;
        \item Exploration rate $\epsilon$ initialised at $1$, and linearly decayed to $0.01$ by multiplying it by $0.995$ after every episode;
        \item Discount factor $\gamma = 0.9$.
    \end{itemize}

    Agents were trained for a total of $10,000$ episodes and were tested every $100$ episodes. The maximum number of timesteps per episode are as follows: 
    \begin{itemize}
        \item $10\times10$: $1,000$ timesteps;
        \item $20\times20$: $4,000$ timesteps.
    \end{itemize}

    \Cref{fig:multipletreasuremaze-pdrm} contains the \pdRM used in this experiment. 

    \begin{figure}[H]
        \centering
        \resizebox{\linewidth}{!}
        {
            \begin{tikzpicture}[shorten >=1pt,auto]
                \node[state,initial]      (u_0)                           {$u_0$};
                \node[state]              (u_1) [right=5cm of u_0]        {$u_1$};
                \node[state]              (u_2) [below=3cm of u_1]        {$u_2$};
                \node[state]              (u_3) [below=3cm of u_2]        {$u_3$};
                \node[state, accepting]   (u_4) [left=5cm of u_2]         {$u_4$};
                \node[state, accepting]   (u_5) [right=5cm of u_3]        {$u_5$};
    
                \path[->]
                    (u_0)   
                        edge [loop above]   node [align=center] {$\{ \mathsf{dir} \}, \PRMStackSymbol / \mathsf{dir}\ \PRMStackSymbol, 0$}    ()
                        edge    node    {$\{ \mathsf{dir}, \mathsf{s} \}, \PRMStackSymbol / \mathsf{dir}\ \PRMStackSymbol, 1\mathrm{e}{3}$}  (u_1)
    
                    (u_1)
                        edge [loop above]   node [align=center] {$ \{ \mathsf{dir} \}, \PRMStackSymbol / \mathsf{dir}\ \PRMStackSymbol, 0$}    ()
                        edge [bend right]   node [left, align=center]    {$\{ \mathsf{dir}, \mathsf{t} \}, \PRMStackSymbol / \mathsf{dir}\ \PRMStackSymbol, 1\mathrm{e}{3}$}   (u_2)

                    (u_2)
                        edge [loop right]   node [align=center] {$\{ \mathsf{\overline{dir}} \}, \mathsf{dir} / \epsilon, 1$}    ()
                        edge [bend right]   node [right, align=center]    {$\{ \mathsf{\overline{dir}}, \mathsf{s} \}, \mathsf{dir} /  \epsilon, 1\mathrm{e}{3}$}   (u_1)
                        edge  node [align=center]    {$\{ \mathsf{\overline{dir}}, \mathsf{s}, \mathsf{T} \}, \mathsf{dir} / \epsilon, 1\mathrm{e}{3}$}    (u_3)
                        edge node [above, align=center]    {$\{ \mathsf{dir}' \}, \mathsf{dir} / \epsilon , -1\mathrm{e}{5}$\\$\{ \mathsf{dir}', \mathsf{s} \}, \mathsf{dir} / \epsilon, -1\mathrm{e}{5}$}    (u_4)
                        
                    (u_3)
                       edge [loop below]   node [align=center] {$\{ \mathsf{\overline{dir}} \}, \mathsf{dir} / \epsilon, 1$}    ()
                       edge node [below left, align=center]    {$\{ \mathsf{dir}' \}, \mathsf{dir} / \epsilon, -1\mathrm{e}{5}$\\$\{ \mathsf{dir}', \mathsf{x} \}, \mathsf{dir} / \epsilon, -1\mathrm{e}{5}$}    (u_4)
                       edge    node    {$\{ \mathsf{\overline{dir}}, \mathsf{x} \}, \mathsf{dir} / \epsilon, 1\mathrm{e}{5}$}  (u_5)
                        ;
            \end{tikzpicture}
        }
        \caption{Pushdown reward machine for the \textsc{MultipleTreasureMaze} domain. $\mathsf{s}$ is the event observed upon reaching the safe location in the maze, $\mathsf{T}$ the one when all treasures have been found. 
        }
        \label{fig:multipletreasuremaze-pdrm}
    \end{figure}

\subsection*{\textsc{DeliverWorld}}

    For this experiment, we report the hyperparameters to train the policy of the Q-learning agent, and the meta-policy and the options' policies of the hierarchical agent. To train these policies, we used the same set of hyperparamters: 

    \begin{itemize}
        \item Learning rate $\alpha = 0.01$;
        \item Exploration rate $\epsilon$ initialised at $1$, and linearly decayed to $0.1$ by multiplying it by $0.9995$ after every episode;
        \item Discount factor $\gamma = 0.99$.
    \end{itemize}
        
    Agents were trained for a total of $10,000$ episodes in the experiment where they performed 8 deliveries both during training and testing episodes, and for $15,000$ episodes in the experiment where they performed 4 deliveries during training and 9 deliveries during testing episodes. In both cases, they were tested every $100$ episodes. For all experiments, each episode (both during training and testing) lasted at most $2,000$ timesteps. At each timestep, the agents received a punishment of $-\frac{1}{2}$ (aside from the \pdRM's reward). 

    \Cref{fig:deliverworld-pdrm} contains the \pdRM used in this experiment. 

    \begin{figure}[H]
        \centering
        \resizebox{\linewidth}{!}
        {
            \begin{tikzpicture}[shorten >=1pt,auto]
                \node[state,initial]      (u_0)                           {$u_0$};
                \node[state]              (u_1) [right=5cm of u_0]        {$u_1$};
                \node[state, accepting]   (u_2) [right=5cm of u_1]        {$u_2$};
    
                \path[->]
                    (u_0)   
                        edge    node    {$\{ \mathsf{loc\_seq}\}, \# / \mathsf{loc\_seq}\ \#, 0$}  (u_1)
    
                    (u_1)
                        edge [loop below]   node [align=center] {$\{ \mathsf{loc} \}, \mathsf{loc} / \epsilon, 1\mathrm{e}{3}$}    ()
                        edge    node    {$\{  \}, \# / \#, 1\mathrm{e}{3}$}  (u_2);
            \end{tikzpicture}
        }
        \caption{Pushdown reward machine for the \textsc{DeliverWorld} domain. $\mathsf{loc\_seq}$ is the sequence of location types given in input at the start of each episode. $\#$ is a special symbol used to denote the start of the stack.
        }
        \label{fig:deliverworld-pdrm}
    \end{figure}

\subsection*{\textsc{WaterWorld}}

    First, for the agents that were trained with ``vanilla'' PPO, we have used the \textsc{Stable-Baselines3}\footnote{\url{https://stable-baselines3.readthedocs.io/en/master/}} implementation of PPO. Instead, for the agent trained with recurrent PPO, we have used the \textsc{SB3 Contrib}\footnote{\url{https://stable-baselines3.readthedocs.io/en/master/guide/sb3_contrib.html}} implementation. \textsc{Stable-Baselines3-Contrib} is the repository extending \textsc{Stable-Baselines3} with community-based implementations of modified or more experimental algorithms.

    We report the set of hyperparameters (using the names of their \textsc{StableBaselines3} variables) that we used for ``vanilla'' PPO and recurrent PPO. We have used a different set of hyperparameters on the two maps ($17\times17$ and $20\times20$). Every hyperparameter that is not reported was set to its default value in \textsc{StableBaselines3} and \textsc{StableBaselines3-Contrib}. Finally, we note that the reported hyperparameters were chosen after tuning both algorithms on both maps. 

    For the two maps and the two agents, we had the following experimental setup. The set of experience samples were collected from $5$ copies of the environment. Each training and testing episode lasted at most $10,000$ timesteps, and agents were evaluated every $100,000$ training timesteps. The learning rate was linearly decayed to $0$ (starting from the ones we will report) in all runs.

    For the $17\times17$ map we have trained both agents for a total of $3,000,000$ timesteps. Following are the hyperparameters for the $17\times17$ map:
    \begin{itemize}
        \item Vanilla PPO:
        \begin{itemize}
            \item $\texttt{learning\_rate} = 3\mathrm{e}{-6}$;
            \item $\texttt{n\_steps} = 10,000$;
            \item $\texttt{batch\_size} = 3,000$;
            \item $\texttt{n\_epochs} = 30$;
            \item $\texttt{gamma} = 0.99$;
            \item $\texttt{gae\_lambda} = 0.9$;
            \item $\texttt{clip\_range} = 0.15$;
            \item $\texttt{clip\_range\_vf} = 1,000$;
            \item $\texttt{vf\_coef} = 0.5$;
            \item $\texttt{target\_kl} = 0.3$;
            \item $\texttt{normalize\_advantage} = \texttt{True}$;
        \end{itemize}
        \item Recurrent PPO:
        \begin{itemize}
            \item $\texttt{learning\_rate} = 3\mathrm{e}{-9}$;
            \item $\texttt{n\_steps} = 10,000$;
            \item $\texttt{batch\_size} = 2,000$;
            \item $\texttt{n\_epochs} = 20$;
            \item $\texttt{gamma} = 0.95$;
            \item $\texttt{gae\_lambda} = 0.95$;
            \item $\texttt{clip\_range} = 0.15$;
            \item $\texttt{clip\_range\_vf} = 1,000$;
            \item $\texttt{vf\_coef} = 1.0$;
            \item $\texttt{target\_kl} = 0.3$;
            \item $\texttt{normalize\_advantage} = \texttt{True}$;
        \end{itemize}
    \end{itemize}

    For the $20\times20$ map we have trained both agents for a total of $5,000,000$ timesteps. Following are the hyperparameters for the $20\times20$ map:
    \begin{itemize}
        \item Vanilla PPO:
        \begin{itemize}
            \item $\texttt{learning\_rate} = 3\mathrm{e}{-5}$;
            \item $\texttt{n\_steps} = 10,000$;
            \item $\texttt{batch\_size} = 3,000$;
            \item $\texttt{n\_epochs} = 30$;
            \item $\texttt{gamma} = 0.99$;
            \item $\texttt{gae\_lambda} = 0.9$;
            \item $\texttt{clip\_range} = 0.15$;
            \item $\texttt{clip\_range\_vf} = 1,000$;
            \item $\texttt{vf\_coef} = 0.5$;
            \item $\texttt{target\_kl} = 0.3$;
            \item $\texttt{normalize\_advantage} = \texttt{True}$;
        \end{itemize}
        \item Recurrent PPO:
        \begin{itemize}
            \item $\texttt{learning\_rate} = 3\mathrm{e}{-6}$;
            \item $\texttt{n\_steps} = 10,000$;
            \item $\texttt{batch\_size} = 2,000$;
            \item $\texttt{n\_epochs} = 20$;
            \item $\texttt{gamma} = 0.95$;
            \item $\texttt{gae\_lambda} = 0.95$;
            \item $\texttt{clip\_range} = 0.15$;
            \item $\texttt{clip\_range\_vf} = 1,000$;
            \item $\texttt{vf\_coef} = 1.0$;
            \item $\texttt{target\_kl} = 0.3$;
            \item $\texttt{normalize\_advantage} = \texttt{True}$;
        \end{itemize}
    \end{itemize}

    \Cref{fig:waterworld-pdrm} contains the \pdRM used in this experiment. 

    \begin{figure}[H]
        \centering
        \resizebox{\linewidth}{!}
        {
            \begin{tikzpicture}[shorten >=1pt,auto]
                \node[state,initial]      (u_0)                           {$u_0$};
                \node[state]              (u_1) [right=5cm of u_0]        {$u_1$};
                \node[state, accepting]   (u_2) [right=5cm of u_1]        {$u_2$};
    
                \path[->]
                    (u_0)   
                        edge [loop below]   node [align=center] {$\{ \mathsf{even} \}, \PRMStackSymbol / \mathsf{ball\_0}\ \PRMStackSymbol, 1\mathrm{e}{3}$\\$\{ \mathsf{odd} \}, \PRMStackSymbol / \mathsf{food\_1}\ \PRMStackSymbol, 1\mathrm{e}{3}$}    ()
                        edge    node    {$\{ \mathsf{A\_F} \}, \PRMStackSymbol / \PRMStackSymbol, 1$}  (u_1)
    
                    (u_1)
                        edge [loop below]   node [align=center] {$\{ \mathsf{ball\_0} \}, \mathsf{ball\_0} / \epsilon, 1\mathrm{e}{3}$\\$\{ \mathsf{ball\_1} \}, \mathsf{ball\_1} / \epsilon, 1\mathrm{e}{3}$}    ()
                        edge    node    {$\{  \}, \# / \#, 1\mathrm{e}{5}$}  (u_2);
            \end{tikzpicture}
        }
        \caption{Pushdown reward machine for the \textsc{WaterWorld} domain. $\mathsf{even}$ and $\mathsf{odd}$ are the events observed when, respectively, the agent captures a ball (except for the first two) which label has the corresponding parity. $\mathsf{A\_F}$ is the event observed when all balls (except for the first two) have been captured. $\#$ is a special symbol used to denote the start of the stack.
        }
        \label{fig:waterworld-pdrm}
    \end{figure}

\end{document}